\documentclass{svjour3}
\usepackage{natbib}
\usepackage{graphicx}
\usepackage{latexsym}
\usepackage{enumerate}
\usepackage{comment}
\usepackage{amssymb,amsmath,amsfonts}
\usepackage{color}
\usepackage{array}
\usepackage{subcaption}
\usepackage[shortlabels]{enumitem}
\usepackage{appendix}
\usepackage{csquotes}
\usepackage{apptools}
\usepackage{rotating}
\usepackage{hyperref}
\usepackage[ruled,vlined,noend]{algorithm2e}

\hypersetup{
   colorlinks=true,
   citecolor=blue,
   linkcolor=blue,
   filecolor=magenta,      
   urlcolor=magenta
}

\journalname{Machine Learning}

\smartqed  

\setlist[description]{font=\bfseries}

\newtheorem{lemmaalt}{Lemma}[lemma]
\newenvironment{lemmap}[1]{
  \renewcommand\thelemmaalt{#1}
  \lemmaalt
}{\endlemmaalt}

\AtAppendix{\counterwithin{lemma}{section}}

\def\urltilda{\kern -.15em\lower .7ex\hbox{\~{}}\kern .04em}

\newcommand{\Mu}{\mathrm{M}}
\newcommand{\Tau}{\mathrm{T}}

\SetKwInput{KwData}{Input}
\SetKwInput{KwResult}{Output}

\SetCommentSty{mycommfont}

\begin{document}
\title{Composition of Relational Features with an Application to Explaining Black-Box Predictors}
\titlerunning{Composition of Relational Features}

\author{
Ashwin Srinivasan$^{\ast}$ \and
A Baskar \and
Tirtharaj Dash$^{\dagger}$ \and
Devanshu Shah}

\authorrunning{A. Srinivasan et al.}

\institute{
    A. Srinivasan \and T Dash \at
    APPCAIR, BITS Pilani, India
    \and
    A. Srinivasan \and A. Baskar \and T Dash \and D Shah\at
    Department of CS \& IS \\
    BITS Pilani, K.K. Birla Goa Campus \\
    Goa 403726, India \\
    \email{\{ashwin,abaskar,tirtharaj,f20180240\}@goa.bits-pilani.ac.in} \\[6pt]
    $^{\ast}$AS is currently visiting TCS Research. 
    He is also a Visiting Professor at the
    Centre for Health Informatics, Macquarie University, Sydney; and
    a Visiting Professorial Fellow at the School of CSE, University of New South Wales, Sydney. \\
    $^{\dagger}$TD is currently at the University of California, San Diego, CA, USA.
}

\date{Received: date / Accepted: date}
\maketitle

\begin{abstract}
Three key strengths of relational machine learning programs like those developed
in Inductive Logic Programming (ILP) are: (1) The use of an expressive subset of first-order
logic that allows models that capture complex
relationships amongst data instances;
(2) The use of
domain-specific relations to guide the construction of models; and
(3) The models constructed are human-readable, which
is often one step closer to being human-understandable. The  price for these advantages is that ILP-like
methods have not been able to capitalise fully on the rapid hardware, software and algorithmic
developments fuelling current developments in deep neural networks.
In this paper, we treat relational features as functions and use the notion of
generalised composition of functions to derive complex functions from simpler ones.
Motivated by the work of McCreath and Sharma \citep{mccreath1998ime,eric:thesis} we formulate
the notion of a 
set of $\Mu$-simple features in a
{\em mode language\/} $\Mu$ and identify
two composition operators ($\rho_1$ and $\rho_2$)
from which all possible
complex features can be derived. We use these results to implement
a form of ``explainable neural network' called
Compositional Relational Machines, or CRMs.  CRMs are labelled
directed-acyclic graphs. The vertex-label for 
any vertex $j$ in the CRM contains a feature-function $f_j$
and an continuous activation function $g_j$.  
If $j$ is a ``non-input'' vertex, then $f_j$ is the composition
of features associated with vertices in
the direct predecessors of $j$.
Our focus is on CRMs in which input vertices (those without
any direct predecessors) all have $\Mu$-simple features in
their vertex-labels. We provide
a randomised procedure for constructing the structure of such CRMs,
and a procedure for estimating the parameters (the $w_{ij}$'s$)$
using back-propagation
and stochastic gradient descent. Using a notion of explanations
based on the compositional structure of features in a CRM,
we provide  empirical evidence on
synthetic data of the ability to identify appropriate explanations;
and demonstrate the use of CRMs as `explanation machines' for
black-box models that do not provide explanations for their
predictions.
\end{abstract}


\section{Introduction}
\label{sec:intro}

It has long been understood that choice of representation can make a significant difference
to the efficacy of machine-based induction. A seminal paper by \cite{quinlan1979discovering} demonstrates
the increasing complexity of models
constructed for a series of problems defined
on a chess endgame, using a fixed representation consisting of 25 features
(called properties in the paper). These
features were identified manually (by him), and captured relations between the
pieces and their locations in the endgame.  Commenting on the increasing complexity
of the models, he concludes: 
\begin{quote}
``This immediately raises the question
of whether these and other properties used in the study were appropriate.
The answer is that they were probably not; it seems likely that a chess expert could
develop more pertinent attributes \ldots If the expert does his job well the induction
problem is simplified; if a non-expert undertakes the definition of properties (as was
the case here) the converse is true.''
\end{quote}
Although Quinlan assumed the representation would be identified manually, the
possibility of automatic identification of an appropriate representation
was already
apparent to \cite{amarel1968representations} a decade earlier: ``An understanding of the relationship between
problem formulation and problem solving efficiency is a prerequisite for the design of
procedures that can automatically choose the most `appropriate' representation of a
problem (they can find a `point of view' of the problem that maximally
simplifies the process of finding a solution)''.
In fact, by `choose' what is probably meant is `construct', if we are to avoid kicking
Feigenbaum's famous knowledge-acquisition bottleneck down the road from extracting models
to extracting representations. 

It has also been known that one way to
construct representations automatically is through the use
of neural networks.
But extraction and re-use of these representations for multiple tasks have become substantially more
common only recently, with the routine availability of specialised hardware. This has allowed
the construction of  networks in which adding layers of computation is assumed to result in
increasingly complex representations (10s of layers are now common, but 100s are also within
computational range). In principle, while a single additional layer
is all that is needed (due to the Universal approximation theorems for neural networks~\citep{HORNIK1989359,Cybenko1989,pinkus_1999}), 
it is thought that the main benefit of the additional
layers lies in constructing representations that are multiple levels of abstractions,
allowing more efficient modelling. There are however 3 important issues that have surfaced:
(1) Automatic construction of abstractions in this way requires a lot of data, often
	100s of 1000s of examples;
(2) The kinds of neural networks used for constructing abstractions depend on the type of data.
For example, certain kinds of networks are used for text, others for images and so on: there is
apparently no single approach for representation learning that can automatically adjust to
the data type; and
(3) The internal representations of data are opaque to human-readability, making it
difficult to achieve the kind of understanding identified by Amarel.

Recent results with neural-based learning suggest
the possibility of viewing representation learning as program synthesis.
Of particular interest are methods like Dreamcoder \citep{ellis2021dreamcoder} that automatically
construct programs for generating data, given some manually identified primitive
functions represented in a symbolic form (Dreamcoder uses $\lambda$-expressions for the
primitive functions).  A combination of generative and discriminative
network models is used to direct the 
hierarchical construction of higher-level functions by compositions
of lower-level ones. Compositions are generated and assessed for utility in
inter-leaved  phases
of generate-and-test (called ``Dream'' and ``Wake'' phases), until the neural-machinery
arrives at a small program, consisting the sequential composition of primitive- and invented functions.
The final result is some approximation to the Bayesian MAP program for generating the data
provided, assuming a prior preference for small programs. There are good reasons to
look at this form of program-synthesis as a mechanism for automated representation learning:
(a) Empirical results with programs like Dreamcoder show that it is possible to identify
	programs with small numbers of examples (the need for large numbers of examples is
	side-stepped by an internal mechanism for generating data in the Dream phase);
(b) In principle, the symbolic language adopted for primitive functions ($\lambda$-expressions) and
	the mechanism of function composition is sufficiently expressive for constructing programs
	for data of any type; 
(c) The intermediate representations have clearly defined interpretations, based on functional
	composition.
There are however some shortcomings. First, the primitive functions have to be manually
identified. Secondly, the construction of new representations requires a combinatorial
discrete search that is usually less efficient than those based on continuous-valued
optimisation. Thirdly, the representation of $\lambda$-expressions, although
mathematically powerful, can prove daunting as a language for encoding domain-knowledge
or for interpreting the results. Finally, the Dreamcoder-like approach for representation learning
has only been demonstrated on very simple generative tasks of a geometric nature.

In this paper, we partially address these shortcomings by drawing on, and extending
some results on representation developed in the area of Inductive Logic
Programming (ILP). The main contributions of this paper are as follows:
\begin{description}
    \item[(a) Conceptual] We develop the conceptual basis for a class of `simple' relational
    features using a well-known specification language
    in ILP (mode-declarations).  Additionally, we develop composition operators for deriving
    more complex relational features, and prove some completeness properties that apply
    to the use of the operators;
    \item[(b) Implementation] We use the concepts developed to specify and implement a form
    of neural network called Compositional Relational Machines, or CRMs. An important feature of
    these networks is that each node is identified with a clearly defined relational feature.
    This allows us to associate structured `explanations' with each node in the network;
    \item[(c) Application] We present empirical results on 2 synthetic data that demonstrate the
    ability of CRMs to construct appropriate explanations; and results on using CRMs to
    act as `explanation machines' for a state-of-the-art black-box predictor on 10 real-world
    datasets.
\end{description}

The rest of the paper is organised as follows:
In \autoref{sec:func} we provide a conceptual framework for
relational features and their compositions. We use this
framework to implement CRMs in \autoref{sec:crm}.
We provide empirical evaluation of CRMs as
explanation machines in \autoref{sec:expt}.
Related work of immediate relevance to this paper are 
presented in \autoref{sec:relworks}. Concluding remarks are in
in \autoref{sec:concl}. The paper has several appendices that
act as supporting material.

\section{Relational Features and their Composition}
\label{sec:func}

In this paper, we are principally interested in specifying and
combining {\em relational features\/}. For us, a $k$-ary relational
feature  will be a function with $k$ terms as arguments. We will
specify such a feature by
$f: {\cal A}_1 \times {\cal A}_2 \times \cdots \times {\cal A}_k  \rightarrow {\cal B}$,
where ${\cal A}_1, {\cal A}_2,\ldots, {\cal A}_k, {\cal B}$ are some sets.
For the most part in this paper, we will restrict ourselves to $k=1$ and
${\cal B} = \{0,1\}$, although the results here can be generalised.
We will denote this setting as $f: {\cal A} \rightarrow \{0,1\}$, for some set ${\cal A}$.
A relational feature is defined in 2 steps. First, we represent the conditions under which
the feature takes the value $1$ using a clause
of the form:\footnote{
See \autoref{app:logic} for a summary of logical syntax and concepts needed for this paper.
We assume a logic with equality ($=/2$).}

\[
    C:~~\forall X ~ (p(X) \leftarrow \exists {\mathbf Y} {Body}(X,{\bf Y}))
\]
or, simply:
\[
 C:~~( p(X) \leftarrow {Body}(X,{\bf Y}))
\]

\noindent
to mean the quantification as shown earlier.
Here, the fixed predicate symbol $p(X)$ is called the {\em head} of $C$,
and ${Body}(X,\cdot)$--the {\em body\/} of clause $C$--is
a conjunction of literals $l_2,l_3\ldots,l_k$ containing some existentially quantified variables
collectively represented here as ${\mathbf Y}$. We assume the body
of $C$ does not contain
a literal of the form $p(\cdot)$ (that is, $C$ is not self-recursive)
and call clauses like these {\em feature-clauses\/}.\footnote{This clause is equivalent to
the disjunct $l_1 \vee \neg l_2 \vee \cdots \vee \neg l_k$. It
will sometimes also be written as the set of
literals $\{l_1, \neg l_2,\ldots,\neg l_k\}$. The literals
can contain predicate symbols representing
relations: hence the term ``relational''.
The requirement that all feature-clauses have the same predicate symbol $p/1$ in
the head is a convenience that will be helpful in what follows. It may be
helpful to read the symbol $p/1$ as a proposition defined on $X$.}
The clausal representation
does not tell us how to obtain the value ($0$ or $1$) of the feature itself for
any $X = a$. For
this we assume an additional set of clausal formulae $B$ (``background'') which
does not contain any occurrence of the predicate-symbol $p/1$
and define the {\em feature-function\/} associated with the clause $C$ as
follows. Let $\theta_a$ denote the substitution $\{X/a\}$ for $a \in {\cal A}$. Then:
 \[f_{C,B}(a) = \left \{
 \begin{array}{l l}
 1~~~~{\mathrm{if}}~ B \cup (C\theta_a) \models p(a) \\
            0~~~{\mathrm {otherwise}}
\end{array}\right.\]

\noindent
In general, given feature-clauses $C_1,C_2,\ldots,C_j$
we will write $f_{C_i,B}(\cdot)$ as $f_i(\cdot)$ ($1 \leq i \leq j)$,
when the context is obvious.
If $f_i(x) = 1$ for $x = a$, we will say ``the feature $f_i$ is true for $x = a$''.

\begin{example}
An early example of a problem requiring relational features was
the problem of discriminating amongst goods trains \citep{michalski1980pattern},
which has subsequently served as a touchstone for the construction
and use of relational features  (see for example, the ``East-West Challenge''
\citep{michie1994international}). In its original formulation, the task is to distinguish
eastbound trains from westbound ones using properties of the carriages and their
loads (the engine's properties are not used), using pictorial descriptions like
these ({\bf {T1}} is eastbound and {\bf {T2}} is westbound):

\vspace*{0.5cm}
\centerline{\includegraphics[width=0.7\textwidth]{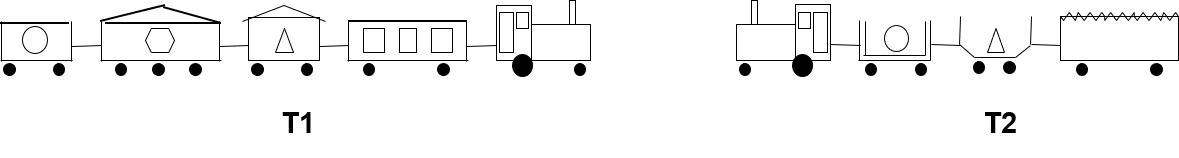}}
\vspace*{0.5cm}

\noindent
Examples of feature-clauses are:
\begin{align*}
    C_1 &: p(X) \leftarrow (has\_car(X,Y), short(Y)) \\
    C_2 &: p(X) \leftarrow  (has\_car(X,Y), closed(Y)) \\
    C_3 &: p(X) \leftarrow  (has\_car(X,Y), short(Y), closed(Y)) \\
    C_4 &: p(X) \leftarrow  (has\_car(X,Y), has\_car(X,Z),  short(Y), closed(Z))
\end{align*}

\noindent
Here, we will assume that predicates like $has\_car/2$, $short/1$, $closed/1$,
$long/1$, $has\_load/3$, etc., are defined as part of the background $B$,
and capture the situation
shown diagrammatically. That is, 
\begin{align*}
    B = \{~ & has\_car(t1,c1\_1), has\_car(t1,c1\_2), \ldots \\
            & long(c1\_1), closed(c1\_1), has\_load(c1\_1,square,3), \ldots \\
            & has\_car(t2,c2\_1), has\_car(t2,c2\_2), \ldots ~\}.
\end{align*}
Then the corresponding feature-function values are:
\begin{align*}
    f_{C_1,B}(t1) = f_1(t1) = 1; &~ f_1(t2) = 1; \\
    f_{C_2,B}(t1) = f_2(t1) = 1; &~ f_2(t2) = 1; \\
    f_{C_3,B}(t1) = f_3(t1) = 1; &~ f_3(t2) = 0; \\
    f_{C_4,B}(t1) = f_4(t1) = 1; &~ f_4(t2) = 1.
\end{align*}
    
Although not used in this paper, feature-clauses
need not be restricted to descriptions of single objects. An
example of a feature-clause about train-pairs for example is:
\begin{align*}
    C_5 &: p((X1,X2)) \leftarrow (has\_length(X1,L1), has\_length(X2,L2), L1 \geq L2).
\end{align*}
\noindent
(The corresponding feature-function will then also be defined over
pairs of objects.)

\end{example}


We intend to describe a mechanism for automatically enumerating feature-clauses like these,
as well as
mechanisms for combining simpler feature-clauses to give more complex ones. We start with
some preliminary definitions needed.

\subsection{Preliminaries}


It will be
necessary for what follows to assume an ordering over literals in a feature-clause.

\begin{definition}[Ordered Clause]
Let $C$ be a feature-clause with 1 head literal and $k-1$ body literals. We assume
an ordering over the literals that maps the set of literals in the clause to a sequence
$\langle C \rangle$ =
$\langle \lambda_1, \lambda_2, \lambda_3,\ldots, \lambda_k \rangle$, where $\lambda_1$
is the head literal and the $\lambda_2,\ldots,\lambda_k$ are literals in the body of the
feature-clause.
\end{definition}

\begin{definition}[Ordered Subclause]
Let $\langle C \rangle$ =
$\langle \lambda_1, \lambda_2, \lambda_3,\ldots, \lambda_k \rangle$ be an ordered
clause. Then an ordered subclause $\langle C' \rangle$ is any clause
$\langle \lambda_1', \lambda_2', ,\ldots, \lambda_j' \rangle$ where
$\lambda_1' = \lambda_1$ and $\langle \lambda_2',\ldots,\lambda_j' \rangle$ is
a sub-sequence of $\langle \lambda_2,\ldots,\lambda_k \rangle$.
\end{definition}

\noindent
From now on, we will use the term ``ordered clause'' to emphasise
an ordering on the literals is assumed. For simplicity, we will assume
that the intended ordering is reflected in a left-to-right reading of the
clause.
Given an ordered clause, it is possible to recover trivially the
set of literals constituting the feature-clause. We use
$Set(\langle C \rangle) = Set(\langle \lambda_1,\ldots,\lambda_k \rangle)$
= $\{\lambda_1,\neg \lambda_2,\ldots,\neg\lambda_k\}$.
Usually, we will further use the set-notation interchangeably with
$\lambda_1 \leftarrow \lambda_2,\ldots,\lambda_k$ to
denote the feature-clause $C$.

\subsection{Feature-Clauses in a Mode Language}
\label{sec:featclauseinmode}
The field of Inductive Logic Programming (ILP) has extensively
used the idea of a {\em mode language\/} to specify a set of acceptable
clauses.\footnote{The notion of associating {\em modes} for predicates
has its origins in typed logics and functional programming. Even in ILP,
modes are not the only way of specifying acceptable sets of clauses; they
are used here as they provide a straightforward way of specifying
the notion of {\em simple features} that follows in a later section.}
We use this approach here to specify the set of feature-clauses.
We provide details of mode-declarations and the definition of a
mode-language based on such declarations in \autoref{app:modes}.
We need the following concepts from the Appendix:
(a) type-names and their definitions;
(b) set of mode declarations and clauses in the mode-language;
(c) input term of type $\gamma$ in some literal; and
(d) output term of type $\gamma$ in some literal.
With these notions in place, we will require the mode-language
for specifying feature-clauses to satisfy the following constraints:

\begin{enumerate}[label=MC{\arabic*}., leftmargin=1.9\parindent]
    \item The set of modes $\Mu$ contains exactly one
        mode-declaration for every predicate occurring in a feature-clause;\label{detmodeconstr}
    \item All modes in $\Mu$ for predicates which appear in the body of a feature-clause
        contain at least 1 ``input'' argument; and
    \item If $\mu=modeh(p) \in \Mu$, then $p$ is an unary predicate and $modeb(p)$ does not occur in $\Mu$.
\end{enumerate}

\noindent
These constraints extend to $p/k$ if the features are defined over 
a product-space.  We note that the restriction MC1 is more strict than the mode language 
allowed by ILP implementations like Progol~\citep{muggleton1995inverse} or 
Aleph~\citep{srinivasan2001aleph}. Effectively, it prevents a predicate being called
in multiple ways, which is allowed in logic programming languages
like Prolog. Here, to achieve the same effect, we will need to
use different predicate symbols for each mode of call. Now, variables (and ground-terms) are constrained by the type-restrictions. Given a set of mode-declarations, feature-clauses in
the mode-language are therefore more constrained than we have presented
thus far.\footnote{
That is,  a feature-clause of
the form $\forall X (p(X) \leftarrow \exists {\mathbf Y}Body(X,{\mathbf Y}))$
should be read as
$\forall X \in \Lambda~~[p(X) \leftarrow \exists {\mathbf Y} \in {\mathbf \Lambda}~Body(X,{\mathbf Y})]$ where
$\Lambda$ and ${\mathbf \Lambda}$ informally denote the sorts of $X$ and
the ${\mathbf Y}$'s. For simplicity, we will not refer to the type-restrictions
on variables and terms when we say that a feature-clause is
in a mode-language. The restrictions are taken as understood, and to be enforced
during inference.} Henceforth we will use $\Mu$ is a set of ``constrained mode-declarations'' to
mean that $\Mu$ satisfies MC1--MC3.




    

\noindent
The categorisation of variables in a literal as being inputs or outputs allows
a natural association of an ordered clause with a graph.

\begin{definition}[Clause Dependency-Graph]
Let $\Mu$ be a set of  constrained mode-declarations and $\Tau$ 
be a set of type definitions for the type-names in $\Mu$.
 Let $\langle C \rangle$  be an ordered clause
 $(\lambda_1\leftarrow \lambda_2,\ldots,\lambda_k)$ in the mode-language
 ${\cal L}_{\Mu,\Tau}$ (see \autoref{app:modes}). 
The clause dependency-graph $G_{\Mu}(\langle C \rangle)$ associated with $\langle C \rangle$ is the labelled directed graph
$(V,E,\psi)$ defined as follows:
\begin{itemize}
    \item $V = \{v_1,v_2\dots,v_k\}$;
    \item for each $i$, $\psi(v_i)=\lambda_i$;
    \item $(v_i,v_j) \in E$ iff:
        \begin{itemize}
            \item $i=1$, $2 \leq j \leq k$, and there exists a variable $X$ such that $\lambda_i$ has $X$ as an input variable of type $\gamma$ in $\Mu$ and $\lambda_j$ has $X$ as an input variable of type $\gamma$ in $\Mu$; or
            \item $1<i<j$, $\lambda_i$ has an output variable $X$ of type $\gamma$ in $\Mu$ and $X$ occurs in $\lambda_j$ as an input variable of type $\gamma$ in $\Mu$.
        \end{itemize}
\end{itemize}
\end{definition}

\begin{example}
\label{ex:hascar}
Let us assume the set of mode-declarations $\Mu$ contain at least the following: $\{$
$modeh(p(+train))$,
$modeb(has\_car(+train,-car))$,
$modeb(short(+car))$,
$modeb(closed(+car))$
$\}$  
where $train$ and $car$ are type-names, with definitions
in $\Tau$. 
The ordered clause
$p(X) \leftarrow (has\_car(X,Y), has\_car(X,Z),  short(Y), closed(Z))$
is in the mode-language ${\cal L}_{\Mu,\Tau}$.
The clause dependency-graph for this ordered clause is given below and
$\psi$ is defined as follows: $\psi(v_1)=p(X)$,
$\psi(v_2)=has\_car(X,Y)$, $\psi(v_3)= has\_car(X,Z)$, $\psi(v_4)=short(Y)$, 
$\psi(v_5)=closed(Z)$. 
\begin{center}
    \includegraphics[width=0.6\linewidth]{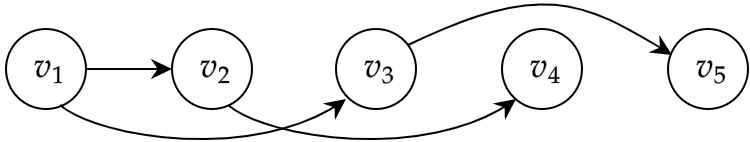}
\end{center}
\end{example}

\begin{remark}
\label{rem:clausegraph}
We note the following about the clause dependency-graph:

\begin{itemize}
\item The clause dependency-graph for an ordered clause is a directed acyclic graph. 
    This is evident from the definition: if $(v_i,v_j) \in E$ then $i<j$.
\item The clause dependency-graph for an ordered clause is unique.
\end{itemize}
\end{remark}

Given a set of modes $\Mu$ we introduce the notion
$\Mu$-simple clauses in a manner similar to \citep{eric:thesis}.

\begin{definition}[Source- and Sink- Vertices and Literals]
\label{def:sourcesink}
Given a set of constrained mode-declarations $\Mu$, type-definitions
$\Tau$, and an ordered clause $\langle C \rangle$ in ${\cal L}_{\Mu,\Tau}$, let
$G_{\Mu}(\langle C \rangle)=(V,E,\psi)$ be the clause dependency-graph of $\langle C \rangle$.
A vertex $v \in V$ is said to be a sink vertex if there is no outgoing edge from $v$. The
corresponding literal, $\psi(v)$, is called a sink literal. A vertex $v \in V$ is said to be a source vertex if there is no incoming edge to $v$. The corresponding literal, $\psi(v)$, is called a source literal. 
\end{definition}

\begin{example}
The clause in \autoref{ex:hascar} has one source 
vertex $v_1$ and two sink vertices: $v_4$ and $v_5$.
Correspondingly, there is one source literal $p(X)$,
and two sink literals: $short(Y)$, $closed(Y)$. 
\end{example}

\noindent

\begin{remark}
\label{rem:singlesource}
Let $\Mu$ satisfy MC1--MC3. Then:
\begin{itemize}
\item The clause dependency-graph of any ordered clause in
    ${\cal L}_{\Mu,\Tau}$ will have exactly 1 source-vertex $v_1$
    (and exactly 1 source-literal). 
\item For every vertex $v$ in the clause dependency-graph,
    there exists at least one path from $v_1$ to $v$.
    Also the union of all the paths from $v_1$ to $v$ will be a directed acyclic
    graph and this is unique. We will denote this directed acyclic graph by ${DAG}_{\langle C \rangle}(v)$.
\item For every vertex $v$ in the clause dependency-graph, either $v$ will be a sink vertex or it     will be on a path from the source vertex ($v_1$) to a sink vertex. 
\end{itemize}
Of these the third observation is not obvious. Suppose a vertex $v$ is not a sink vertex.
Then it will have at least one outgoing edge from it. By following outgoing edges forward,
we will end  in a sink vertex. If $v\neq v_1$, then there is at least one incoming edge to $v$. By following incoming edges backward, we will end in a source vertex. Since there is only one source vertex, this will be $v_1$.  So $v$ will be a sink vertex or it will be on a path from $v_1$ to a sink vertex.
\end{remark}

\begin{definition}[$\Mu$-Simple Feature-Clause]
Given a set of constrained mode-declarations $\Mu$, type-definitions $\Tau$,
an ordered feature-clause $\langle C \rangle$ in the mode-language ${\cal L}_{\Mu,\Tau}$
is said to be an $\Mu$-simple feature-clause, or simply a $\Mu$-simple clause iff the clause dependency-graph
$G_{\Mu}(\langle C \rangle)$ has exactly one sink literal.  
\end{definition}

\begin{example}
We continue \autoref{ex:hascar}.
The ordered clauses $p(X) \leftarrow has\_car(X,Y)$, $short(Y)$  and $p(X) \leftarrow has\_car(X,Y)$ are $\Mu$-simple clauses as
both have only one sink literal  $short(Y)$ and $has\_car(X,Y)$ respectively. 
The ordered clause
$p(X) \leftarrow has\_car(X,Y), short(Y), closed(Y)$ is not a $\Mu$-simple clause
as it has two sink literals $short(Y)$ and $closed(Y)$. 
\end{example}

\begin{definition}[Maximal $\Mu$-simple subclause]
Given a set of constrained mode-declarations $\Mu$, type-definitions $\Tau$,
an ordered clause $\langle C' \rangle$ in the mode-language ${\cal L}_{\Mu,\Tau}$
is said to be a maximal $\Mu$-simple subclause of an ordered clause
$\langle C \rangle$ in ${\cal L}_{\Mu,\Tau}$ iff: 
(a) $\langle C'\rangle$ is an ordered subclause of $\langle C\rangle$; and
(b) there is an isomorphism  from the clause dependency-graph
    $G_{\Mu}(\langle C' \rangle)$ to 
    ${DAG}_{\langle C \rangle}(v)$ for some sink vertex $v$ in $G_{\Mu}(\langle C \rangle)$.
\end{definition}

\begin{example}
Continuing \autoref{ex:hascar},
the ordered clause $p(X) \leftarrow has\_car(X,Y), short(Y)$ is a maximal $\Mu$-simple
subclause of
$p(X) \leftarrow has\_car(X,Y), short(Y), closed(Y)$. The ordered clause
$p(X) \leftarrow has\_car(X,Y)$ is not a maximal  $\Mu$-simple subclause of
$p(X) \leftarrow has\_car(X,Y), short(Y),closed(Y)$. 
\end{example}

\begin{definition}[Basis]
Let $\Mu$ be a set of constrained mode-declarations, $\Tau$ be a set of type-definitions, $\langle C \rangle$
be an ordered clause  in the mode-language ${\cal L}_{\Mu,\Tau}$.
Then $Basis(\langle C \rangle)$ = $\{$
$\langle C' \rangle ~|~ \langle C' \rangle $ is a maximal $\Mu$-simple subclause of $\langle C \rangle$ $\}$.
\end{definition}

\begin{example} The basis for $p(X) \leftarrow has\_car(X,Y), has\_car(X,Z), short(Y), closed(Z)$ is $\{p(X) \leftarrow has\_car(X,Y), short(Y)$, $p(X) \leftarrow has\_car(X,Z),closed(Z)\}$.
\end{example}

\begin{remark}
For given an ordered clause $C$ in ${\cal L}_{\Mu,\Tau}$, $Basis(\langle C \rangle)$ is unique. Moreover, if the number of sink vertices in the clause dependency-graph of $\langle C \rangle$ is $k$, then $|Basis(\langle C \rangle)|=k$. 
\end{remark}

\begin{lemma}[Basis Lemma]
\label{lemma:basis}
Let $\Mu$ be a set of constrained mode-declarations, $\Tau$ be a set of type-definitions. Let
 $\langle C \rangle$ be an ordered clause in the mode-language ${\cal L}_{\Mu,\Tau}$ with
 $k$ sink-literals. If $Basis(\langle C \rangle) = \{
 \langle S_1 \rangle,\langle S_2\rangle,\dots, \langle S_k \rangle\}$ then  
$\bigcup_{i=1}^k S_i= C$. 

\end{lemma}
\begin{proof}
Let $G_{\Mu}(\langle C \rangle)=(V,E,\psi)$ be the clause dependency-graph for the ordered clause
$\langle C \rangle$ and $Basis(\langle C \rangle)=$
$\{$
$ \langle S_1 \rangle,\langle S_2\rangle,\dots, \langle S_k \rangle\}$. 
We prove $\bigcup_{i=1}^k S_i \subseteq C$ and $C \subseteq \bigcup_{i=k}^k S_i$.
We consider first  $\bigcup_{i=1}^k S_i \subseteq C$. Assume the contrary. That is,
there exists some $l \in \bigcup_{i=1}^k S_i$ but $l \not\in C$. 
Since $l$ is a literal in $\bigcup_{j=1}^k S_j$, then $l \in S_j$ for some $j$.
Since every $\langle S_j \rangle$ is an ordered subclause of $\langle C \rangle$, by
definition every literal in $S_j$ occurs in $C$. Therefore $l \in C$
which is a contradiction. 

Next we consider $C \subseteq \bigcup_{i=1}^k S_i$. Let $l$ be a literal in $C$. There exists a vertex $v_j$ in the clause dependency-graph $G_\Mu(\langle C \rangle)$ such that $\psi(v_j)=l$. Either $v_j$ is a sink vertex or not a sink vertex in $G_\Mu(\langle C \rangle)$. If it is a sink vertex, then there exists a maximal $\Mu$-simple subclause $\langle S_j \rangle $ with $v_j$ as a sink vertex. Hence $l$ is in $S_j$. If $v_j$ is not a sink vertex, then
it will be on the path from $v_1$ to some sink vertex $v_m$ (see \autoref{rem:singlesource}). Then the directed acyclic sub-graph $DAG_{\langle C \rangle}(v_m)$  will have this vertex $v_j$. Since $v_m$ is a sink vertex, there exists a 
maximal $\Mu$-simple subclause $\langle S_m \rangle $ with $v_m$ as a sink vertex and there is an isomorphism between the clause dependency-graph $G_\Mu(\langle S_m\rangle)$ and 
${DAG}_{\langle C \rangle}(v_m)$. 
Hence
$l \in S_m$. So in both cases $l$ is in $\bigcup_{i=1}^k S_i$. Hence $C
\subseteq \bigcup_{i=1}^k S_i$. \qed
\end{proof}

Let $\Mu$ be a set of constrained mode-declarations, and $\Tau$ be a set of
type definitions. Let $\Mu'$ be $\Mu$ extended
with an additional mode-declarations allowing body-literals of the form
$+\gamma = +\gamma$
(that is, $\Mu'$ allows equality between variables of the same type $\gamma$,
if this is not already allowed in $\Mu$); 
the definition of $=/2$  is provided by axioms of the equality logic.
For more details see \autoref{app:logic}.
 We note that if the ordered clause $\langle C \rangle$ is
 in ${\cal L}_{\Mu,\Tau}$, then $\langle C \rangle$ is in  ${\cal L}_{\Mu',\Tau}$.





We define operators $\rho_1,\rho_2$ as follows:
\begin{enumerate}
    \item  Let $\langle C \rangle$ be in ${\cal L}_{\Mu',\Tau}$ s.t.
            $\langle C \rangle: p(X) \leftarrow Body(X,{\mathbf{Y}})$. Then
            $\rho_1(\langle C \rangle)$ = $\{$
                $p(X) \leftarrow Body(X,{\mathbf Y}), Y_1 = Y_2 ~|~$
                    $Y_1, Y_2$ are output variables of the same type in $Body$ $\}$;
    \item  Let $\langle {C}_1 \rangle$, $\langle {C}_2 \rangle$ be in ${\cal L}_{\Mu',\Tau}$ s.t.
            $\langle {C}_1 \rangle: p(X) \leftarrow {Body}_1(X,{\mathbf{Y_1}})$
            and
             $\langle {C}_2 \rangle: (p(X) \leftarrow {Body}_2(X,{\mathbf{Y_2}})$.
             Then 
              $\rho_2(\langle {C}_1 \rangle, \langle {C}_2 \rangle)$ = $\{$
                $p(X) \leftarrow {Body}_1(X,{\mathbf Y_1}), {Body}_2(X,{\mathbf Y_2})$
                $\}$
\end{enumerate}

\noindent
These operators allows us to establish a link between the derivability
of clauses in ${\cal L}_{\Mu',\Tau}$ using $\{\rho_1,\rho_2\}$
and clauses in ${\cal L}_{\Mu,\Tau}$.

\begin{definition}[Derivation of Feature-Clauses] 
Let $\Mu$ be a set of mode-declarations, and $\Mu'$ be an extension
of $\Mu$ as above. Let $\Tau$ be a set of type-definitions, and 
$\Omega \subseteq  \{\rho_1,\rho_2\}$.
Let ${\Phi}$ be a set  of feature-clauses in ${\cal L}_{\Mu',\Tau}$.
A sequence of feature-clauses $\langle C_1 \rangle ,\langle C_2\rangle,\dots,\langle C_n
\rangle $ is said to be a derivation sequence of $\langle C_n \rangle $ from $\Phi$ using $\Omega$ iff each clause $
\langle C_i \rangle $ in the sequence is either :
(a) an instance of an element of $\Phi$ such that 
no variables other than  $X$ occur earlier in this sequence; or 
(b) an element of the set $\rho_1(\langle C_j \rangle)$ ($j < i$), if $\rho_1 \in \Omega$; or
(c) an element of the set $\rho_2(\langle C_j \rangle, \langle C_k \rangle)$ ($j,k < i$), if
    $\rho_2 \in \Omega$. We will say $\langle C_n \rangle $ is derivable
from $\Phi$ using $\Omega$  if there exists a derivation sequence of $\langle C_n \rangle$ from $\Phi$ using $\Omega$.
\end{definition}
 
\begin{example}\label{ex:derivation}
Let us assume the set of mode-declarations $\Mu$ contain the following: 
\begin{eqnarray*}
    \{~
        modeh(p(+train)), modeb(has\_car(+train,-car)), modeb(short(+car)),\\
        modeb(closed(+car)), modeb(smaller(+car,+car)
    ~\},
\end{eqnarray*}
where $train$ and $car$ are type-names, with definitions in $\Tau$. Here is a derivation 
sequence of 
\begin{align*}
    p(X) \leftarrow &~ has\_car(X,U),has\_car(X,V),smaller(U,V),U=V,\\
         &~ has\_car(X,Y),short(Y),U=Y
\end{align*}  
from 
\begin{align*}
    \{~ p(X) &\leftarrow has\_car(X,Y),short(Y), \\
        p(X) &\leftarrow has\_car(X,U),has\_car(X,V),smaller(U,V) ~\}
\end{align*}
using $\{\rho_1,\rho_2\}$. 

\begin{tabbing}
1 ~~\= $p(X) \leftarrow has\_car(X,U),has\_car(X,V),smaller(U,V)$~~~~~~~~~~~~~~~~~~~~~~~\= Given\\
2 \> $p(X) \leftarrow has\_car(X,U),has\_car(X,V),smaller(U,V),U=V$\> $1,\rho_1,U,V$\\
3 \> $p(X) \leftarrow has\_car(X,Y),short(Y)$\> Given\\
4 \> $p(X) \leftarrow has\_car(X,U),has\_car(X,V),smaller(U,V),U=V,$ \> \\
\>~~~~~~~~~~~$has\_car(X,Y),short(Y)$ \> $2,3,\rho_2$\\
5 \> $p(X) \leftarrow has\_car(X,U),has\_car(X,V),smaller(U,V),U=V,$\>\\
\>~~~~~~~~~~~$has\_car(X,Y),short(Y),U=Y$ \> $4,\rho_1,U,Y$\\
\end{tabbing}

\end{example}

It is useful to define the notion of a
$\rho$-derivation graph from a set of feature-clauses $\Phi$.

\begin{definition}[$\rho$-derivation graph given $\Phi$]
\label{def:dergraph}
Let $\gamma = (V,E,\phi)$ be  a labelled DAG with vertices $V$,
edges $E$ and vertex-labelling function $\phi$. Let $Pred(v)$ denote
the set of immediate predecessors of any $v \in V$.  
Let ${\cal F}_{\Mu}$ be a set of feature-clauses given modes
$\Mu$ and $I \subseteq {\cal F}_{\Mu}$. Let ${\Phi}$ be a set  of feature-clauses in ${\cal L}_{\Mu',\Tau}$.
Then $\gamma$ is a $\rho$-derivation graph given $\Phi$
iff:
    \begin{itemize}
        \item For each vertex $v_i \in V$, $\phi(v_i) = C_i$, where
            $C_i \in {\cal F}_{\Mu}$;
        \item $0 \leq |Pred(v)| \leq 2$ for all $v \in V$;
         \item For each $v \in V$:
            \begin{itemize}
                \item If $Pred(v) = \emptyset$ then $\phi(v) \in \Phi$;
                \item If $Pred(v) = \{u\}$ then $\phi(v) \in \rho_1(\phi(u))$;
                \item If $Pred(v) = \{u_1,u_2\}$ then $\phi(v) \in \rho_2(\phi(u_1),\phi(u_2))$
            \end{itemize}
    \end{itemize}
\end{definition}

\noindent
Since we are only concerned with $\rho_1,\rho_2$ in this
paper, we will usually call this the derivation graph given $\Phi$
or even just the derivation graph, when $\Phi$ is understood. 

\begin{remark}
\label{rem:rhoequiv}
We note that $\rho_1$ and $\rho_2$ preserve equivalence, in the
following sense:
    \begin{itemize}
        \item If $C' \equiv C$ then  $\rho_1(C) \equiv \rho_1(C')$; and
        \item If $C'_1 \equiv C_1$ and $C_2' \equiv C_2$ then
            $\rho_2(C1,C2) \equiv \rho_2(C_1',C_2')$
    \end{itemize}
    Here, equivalence across sets has the usual conjunctive meaning.
    That is, for sets $A,B$, $A \equiv B$ iff
    $\underset{x \in A}{\bigwedge} x ~\equiv~\underset{y \in B}{\bigwedge} y$. Two ordered clauses $\langle C_1 \rangle $ and $\langle C_2 \rangle$ are equivalent iff the $Set(\langle C_1 \rangle)$ is equivalent to the $Set(\langle C_2\rangle)$. 

\end{remark}

\begin{definition}[Closure]
Let $\Phi$ be a set of feature-clauses and $\Omega \subseteq \{\rho_1,\rho_2\}$. We define the closure of $\Phi$ using $\Omega$  as the set
of ordered clauses $\langle C 
\rangle $ which has a derivation sequence from $\Phi$ using $\Omega$. 
We use $Closure_\Omega(\Phi)$ to denote the closure of $\Phi$ using $\Omega$.
\end{definition}

    


We will say $\theta$ is a type-consistent substitution if for every variable $U$, the substitution $U/t \in \theta$ (that is, $\theta(U) = t$), then $U,t$ have
the same type in $\Mu$. It follows that if $\theta$ is a type-consistent
substitution for variables in an ordered clause $\langle C \rangle$ in ${\cal L}_{\Mu,\Tau}$  and
$\theta(u) = \theta(v)$ for $u,v$ in $\langle C \rangle$, then 
$u,v$ have the same type in $\Mu$.


\begin{lemma}[Derivation Lemma]
\label{lemma:deriv}
Given $\Mu, \Mu'$ and $\Omega = \{\rho_1,\rho_2\}$ as before. Let
$\langle C \rangle$ be an ordered clause in ${\cal L}_{\Mu,\Tau}$, with
head $p(X)$.
Let $S$ be a set of ordered $\Mu$-simple clauses in ${\cal L}_{\Mu,\Tau}$, with
heads $p(X)$ and all other variables of clauses in $S$ standardised apart from each other and from $C$. If there
exists a substitution $\theta$ s.t.
$Basis(\langle C \rangle) \subseteq S\theta $ then
there exists an 
ordered clause $\langle C' \rangle$
in ${\cal L}_{\Mu',\Tau}$  such that $\langle C' \rangle$ is equivalent to $\langle C \rangle$ and derivable from $S$ using $\Omega$. 
\end{lemma}

\begin{proof}
See \autoref{app:deriv}. \qed
\end{proof}

\begin{remark}
\label{rem:subsume}
Let $\langle C_1 \rangle,\langle C_2 \rangle, \ldots, \langle C_n \rangle$
be a derivation from a set of ordered clauses $S$ using $\{\rho_1,\rho_2\}$.
 Also, for any clause $\langle C_i \rangle$ in the derivation
sequence, let $f_i$ denote the corresponding feature-function
as defined in \autoref{sec:func} using background knowledge $B$. Let $a$ denote a
    data-instance. We note the following consequences for $1 \leq i < j  \leq n$:
    \begin{enumerate}
        \item $C_i$ subsumes $C_j$;\footnote{
            Here subsumption is used in the sense described by 
            \cite{plotkin1972automatic}. That is, representing clauses 
            as sets of literals, clause $C$ subsumes
            clause $D$ iff there exists a substitution
            $\theta$ s.t. $C\theta \subseteq D$.}
        \item If $f_i(a) = 0$ then $f_j(a) = 0$;
        \item If $f_j(a) = 1$ then $f_i(a) = 1$; and
        \item If $\langle C_{i+1} \rangle \in \rho_1(\langle C_i \rangle)$,
                $f_{i}(a) = 1$ and $f_{i+1}(a) = 0$ then
                there exists a clause
                $C_{i+1}'$ s.t. $C_i \equiv C_{i+1} \vee C_{i+1}'$
                and $f_{B,C_{i+1}'}(a) = 1$
    \end{enumerate}
 \noindent
 (1) follows straightforwardly since for any $\langle C_i \rangle$,
subsequent clauses in the derivation only result in
the addition of literals (that is, $C_i \subset C_j$ for $i < j$).
For (2), we note that
since $C_i \subset C_j$ and both $C_i, C_j$ have the same head literal ($p(X)$)
we can take $C_i =  \forall (p(X) \vee l_1 \vee \cdots l_k)$
and $C_j = \forall (p(X) \vee l_1 \vee \cdots l_k \vee l_{k+1} \vee \cdots l_m)$.
If $f_i(a) = 0$ then
$B \cup C_i \not\models p(a)$.
That is, there exists some interpretation
$I$ that is a model for $B \cup C_i$ s.t. $p(a)$ is false in $I$.
If $I$ is a model for $B \cup C_i$ then it is a model for $C_i$.
Further, if $I$ is a model for $C_i$ and $p(a)$ is false in $I$ 
then $I$ is a model for $\forall(l_1 \vee \cdots l_k)$.
But then $I$ is a model for
$C_j = \forall  (p(X) \vee l_1 \vee \ldots l_k \vee l_{k+1} \vee \cdots l_m)$.
Thus $I$ is a model for $B \cup C_j$ and not a model for $p(a)$.
That is, $B \cup C_j \not\models p(a)$ and $f_j(a) = 0$.
(3) follows from the
fact that if $C_i$ subsumes $C_j$ then $C_i \models C_j$ \citep{gottlob1987subsumption}.
Therefore, if $B\cup C_j \models p(a)$ then $B\cup C_i \models p(a)$.
That is, if $f_j(a) = 1$ then $f_i(a) = 1$.
For (4), let $C_i: p(X) \leftarrow {Body}_i(X,{\mathbf Y})$.
Then $C_{i+1}: p(X) \leftarrow {Body}_i(X,{\mathbf Y}), Y_i = Y_j$ for
$y_{i,j} \in {\mathbf Y}$. Since $f_i(a) = 1$ and $f_{i+1}(a) = 0$,
it must be the case that $Y_i = Y_j$ does not hold for $x = a$. Let
$C_{i+1}': p(X) \leftarrow {Body}_i(X,{\mathbf Y}), Y_i \neq Y_j$.
It is evident that, $f_{B,C_{i+1}'}(a) = 1$ and
$C_i \equiv C_{i+1} \vee C_{i+1}'$.
\end{remark}

\noindent
A specialised form of derivation results from the repeated use of $\rho_2$ first, 
followed by the repeated use of $\rho_1$. We call this form of derivation
a {\em linear derivation\/}. We describe this next (relevant proofs
are in \autoref{app:linderiv}).

\begin{definition}[Linear Derivation of Feature-Clauses]
Let $\Mu$ be a set of mode-declarations, and $\Mu'$ be an extension
$\Mu$ as earlier. Let $\Tau$ be a set of type-definitions.
Let ${\Phi}$ be a set  of feature-clauses in ${\cal L}_{\Mu',\Tau}$
and $\rho_1,\rho_2$ be the operators defined earlier.
A sequence of feature-clauses $\langle C_1 \rangle ,\langle C_2 \rangle,\dots,\langle C_n \rangle $ is said to be a linear derivation sequence of $\langle C_n \rangle $ from $\Phi$ using 
$\{\rho_1,\rho_2\}$ iff there exists $j$ such that $1 \leq j \leq n$ and:
\begin{itemize}
    \item For $i\leq j$:
        \begin{itemize}
            \item Clause $\langle C_i \rangle $ in the sequence
                is either an element of $\Phi$ or an element of the set $\rho_2(\langle C_{i-1}\rangle,\langle C_k \rangle)$ where $\langle C_k \rangle \in \Phi$ and $k<i$.
        \end{itemize}
    \item For $i > j$:
        \begin{itemize}
            \item Clause $\langle C_i \rangle $ is an element of the     set $\rho_1(\langle C_{i-1} \rangle)$.
        \end{itemize}
\end{itemize}
We will say $\langle C_j \rangle$ is linearly derivable from $\Phi$ using $\{\rho_2\}$; and
$C_n$ is linearly derivable from $\Phi$ using $\{\rho_1,\rho_2\}$.
\end{definition}

\begin{example}
We continue \autoref{ex:derivation}.  Below is a linear derivation 
sequence of 
\begin{align*}
    p(X) \leftarrow &~ has\_car(X,U),has\_car(X,V),smaller(U,V),\\
                    &~ has\_car(X,Y),short(Y),U=V, U=Y
\end{align*}
from
\begin{align*}
    \{~  p(X) &\leftarrow has\_car(X,Y),short(Y), \\
        p(X) &\leftarrow has\_car(X,U),has\_car(X,V),smaller(U,V) ~\}
\end{align*}
using $\{\rho_1,\rho_2\}$. 

\begin{tabbing}
1 ~~\= $p(X) \leftarrow has\_car(X,U),has\_car(X,V),smaller(U,V)$~~~~~~~~~~~~~~~~~~~~~~~\= Given\\
2 \> $p(X) \leftarrow has\_car(X,Y),short(Y)$\> Given\\
3 \> $p(X) \leftarrow has\_car(X,U),has\_car(X,V),smaller(U,V)$ \> \\
\>~~~~~~~~~~~$has\_car(X,Y),short(Y)$ \> $1,\rho_2$\\
4 \> $p(X) \leftarrow has\_car(X,U),has\_car(X,V),smaller(U,V),$ \> \\
\>~~~~~~~~~~~$has\_car(X,Y),short(Y),U=V$ \> $3,\rho_1, U,V$\\
5 \> $p(X) \leftarrow has\_car(X,U),has\_car(X,V),smaller(U,V),$\>\\
\>~~~~~~~~~~~$has\_car(X,Y),short(Y),U=V, U=Y$ \> $4,\rho_1,U,Y$\\
\end{tabbing}

There is no way to derive the clause $p(X) \leftarrow has\_car(X,U),has\_car(X,V),$ $smaller(U,V),$ $U=V, has\_car(X,Y), short(Y), U=Y$  using linear derivation from the given set, but we can derive an equivalent clause $p(X) \leftarrow has\_car(X,U),has\_car(X,V),$ $smaller(U,V),$ $has\_car(X,Y)$, $short(Y), U=V, U=Y$ using linear derivation. We would like to point out that the positions of equality literal $U=V$ in the first clause and the second clause are different.
\end{example}

\begin{lemma}[Linear Derivation Lemma]
\label{lemma:linderiv}
Given $M,M'$ and a set
of ordered clauses $\Phi$. If an ordered clause $\langle C \rangle$ is derivable from $\Phi$ using $\{\rho_1,\rho_2\}$
then there exists an equivalent ordered clause $\langle C' \rangle$ and it is linearly derivable from $\Phi$ using $\{\rho_1,\rho_2\}$.
\end{lemma}

\begin{proof}
See \autoref{app:linderiv}. \qed
\end{proof}

Feature-clauses and their composition using the
$\rho$-operators provide the tools for the development
of a particular kind of neural network that we describe next.

\section{Compositional Relational Machines (CRMs)}
\label{sec:crm}

Formally, a CRM is defined as follows:

\begin{definition}[CRM]
\label{rem:crm}
A CRM is a 7-tuple $(V,I,O,E,\phi,\psi,h)$ where:
\begin{itemize}
\item $V$ denotes a set of vertices; 
\item $I \subseteq V$ is a set of
``input'' vertices; 
\item $O \subseteq V$ is a set of ``output''
vertices; 
\item  $E \subseteq \{(v_i,v_j): v_i,v_j \in V, v_i \not\in O, v_j \not\in I \}$;
\item A vertex-labelling function
         $\phi: V \rightarrow {\cal F}_\Mu \times {\cal G}$, where
            ${\cal F}_{\Mu}$ is the set of feature-clauses given
            a set of modes $\Mu$; and $\cal G$ denotes a set of
            activation functions.\footnote{We assume activation
            functions in ${\cal G}$ are $\mathbb{R} \rightarrow \mathbb{R}$.}
            In this paper, we will further assume, if $v \in I$ then we
            restrict
            $\phi(v) = (\cdot,{\mathbf{1}})$, where ${\mathbf{1}(\cdot)}$ = 1; 
        \item An edge labelling function $\psi:E \rightarrow \mathbb{R}$, assigns       some real-valued labels to edges in the graph; and
        \item $h:{\mathbb{R}}^{|O|}\to {\mathbb{R}}^{n}$ is a computation
            function, for some fixed $n$
            \end{itemize}
    such that $(V,E,\phi')$ is a derivation graph (\autoref{def:dergraph}) given $\Phi$ where
            $\phi'(v)= C $ if $\phi(v)=(C,\cdot) $ and
            $\Phi=\{\phi'(v)~|~v \in I\}$.
\end{definition}

We note 2 important features of CRMs: (1) Each vertex has a
feature-clause associated with it; and (2) Edges between vertices
in a CRM are required to satisfy the constraints on edges imposed
by a derivation graph. That is, the only edges allowed
are those that result from $\rho_1$ or $\rho_2$ operations
on the feature-clauses associated with the vertices.

\subsection{CRMs as Explainable Neural Networks}
\label{sec:crmnn}

We describe a use of CRMs as a form of neural network
capable of generating logical explanations relevant
to its prediction. The architecture of the
neural network is inspired by Turing's idea of unorganised machines~\citep{turing1948intelligent}
(see \autoref{fig:crmnode}). Each ``neuron''
has 2 parts, implementing the vertex-label specification of
a CRM's node:
(i) An arithmetic part that is concerned with the $g$-function
in the CRM's vertex-label; and
(ii) A logical part that acts as a switch,
depending on the feature-clause associated with
the CRM's vertex-label.  We call neurons
in such a network ``arithmetic-logic neurons'' or ALNs
for short.

\begin{figure}[!htb]
    \centering
    \includegraphics[width=0.5\linewidth]{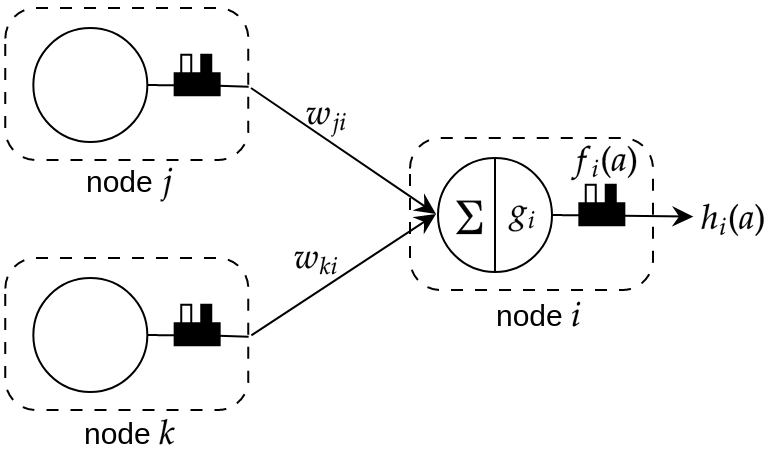}
    \caption{A neural network implementation of a CRM inspired
        by Turing's gated neural networks \citep{turing1948intelligent}. A
    neuron $n_i$ corresponds to a vertex $v_i$ in the CRM,
    with vertex-label $\psi(v_i) = (C_i,g_i)$.
    In the figure, $n_i$ is connected to neurons $n_j$ and
    $n_k$, implying $(v_j,v_i)$ and $(v_k,v_i)$ are in the
    edge-set of the CRM.
    $f_i$ (correctly $f_{C_i,B})$ is
    the feature-function obtained from the feature-clause
    $C_i$ (see \autoref{sec:func}), and acts as a gate.
       For a data instance $a$,
     $g_i(w_{ji}h_j(a) + w_{ki}h_k(a))$ passes through the gate if and
     only $f_i(a) = 1$. In general, $h_i(a)$ is thus
     $g_i(\sum_{k=Pred(n_i)}w_{ki}h_k(a))$ or $0$, where
     $Pred(n_i)$ is the set of immediate predecessors of $n_i$
     in the neural network.}
    \label{fig:crmnode}
\end{figure}

In the rest of the paper, we will use ``CRM'' synonymously
with this form of neural-network implementation.
The 7-tuple defining a CRM $(V,I,O,E,\phi,\psi,h)$ corresponds
to the following aspects of the neural implementation:
(a) The structure of the network is defined by $V,I,O,E,$ and $\phi$;
(b) The parameters of the network are defined by $\psi$; and
(c) the computation of the network is defined by $h$.
We consider each of these in turn.

\subsubsection{Structure Selection}
\label{sec:struc}

\autoref{alg:crm} is an enumerative procedure for
obtaining a 5-tuple $(V,I,O,E,\phi$), given a set of
feature-clauses $\Phi$.
For simplicity, the procedure assumes
a single activation function $g$.

\begin{algorithm}[!htb]
\SetAlgoNoLine
\LinesNumbered
\KwData{A set $\Phi$ of feature-clauses
         $\{C_1,C_2,\ldots,C_n\}$ with heads p(X) and all other variables are standardised apart from each
other;  
    an activation function g; and
    a bound $d$ on the depth of composition}
\KwResult{A CRM structure}
$I = \{v_1,v_2,\ldots,v_n\}$\;
$V = I$\;
$E = \emptyset$\;
\For{$i = 1$ to $n$}{
        $\phi(v_i) = (C_i,\mathbf{1})$
   }
\For{$j = 1$ to $d$}{
        $V_{j,1} = \{({C'},v): v \in V, \phi(v) = (C,\cdot), C' = \rho_1(C) \}$\;
        $V_{j,2} = \{(C',v_1,v_2): v_1,v_2 \in V, \phi(v_1) = (C_1,\cdot), \phi(v_2) = (C_2,\cdot), C' = \rho_2(C_1,C_2)\}$\;
        Let $V_1$ be a set of new vertices $v$ s.t. there exists
            $(C',v') \in V_{j,1}$ and $\phi(v) = (C',g)$\;
         Let $V_2$ be a set of new vertices $v$ s.t. there exists
           $(C',v_1,v_2) \in V_{j,2}$ and $\phi(v) = (C',g)$\;
         Let $V' = V_1 \cup V_2$\;
        $V = V \cup V'$\;
        $E = E \cup \{(v_1,v):~(C,v_1) \in V_{j,1}, \phi(v) = (C,g)\}$
               $\cup$
               $\{(v_1,v),(v_2,v):~(C,v_1,v_2) \in V_{j,2}, \phi(v) = (C,g)\}$\;
}
$O = \{v: v \in V, \nexists v' ~s.t.~ (v,v') \in E \}$\;
\Return $(V,I,O,E,\phi)$
\caption{\textbf{(ConstructCRM)} Depth-bounded construction
    of a CRM's structure}
\label{alg:crm}
\end{algorithm}

\autoref{alg:crm} has an important practical difficulty:
\begin{itemize}
    \item We are interested in a class of CRMs that can be constructed
        using a set of $\Mu$-simple feature-clauses
        $\{C_1,\ldots,C_n\}$. Now, it may be
        impractical to obtain all possible
        $\Mu$-simple feature-clauses in a mode-language. Even if
        this were not the case, it may
        be impractical to derive all non-simple clauses in the
        manner shown in \autoref{alg:crm}.
\end{itemize}

\autoref{alg:randcrm} describes a randomised implementation to
address this. The procedure also uses the result in the Linear
Derivation Lemma (\autoref{lemma:linderiv} in \autoref{sec:func}) to construct
a CRM structure that first uses the
$\rho_2$ operator, followed by the $\rho_1$ operator.

\begin{algorithm}[!htb]
\SetAlgoNoLine
\LinesNumbered
\KwData{A set $\Phi$ of feature-clauses
         $\{C_1,C_2,\ldots,C_n\}$ with heads p(X) and all other variables are standardised apart from each
other;  
    an activation function $g$; 
    a sample size $s$; 
    bounds $d_{\rho_1}, d_{\rho_2}$ on the depth of
    application of the $\rho_1$ and $\rho_2$ operators}
\KwResult{A CRM structure}
$I = \{v_1,v_2,\ldots,v_n\}$\; 
$V_0 = I$\;
$V = V_0$\;
$E = \emptyset$\;
\For{$i = 1$ to $n$}{
    $\phi(v_i) = (C_i,\mathbf{1})$\;
}
\For{$j = 1$ to $d_{\rho_1} + d_{\rho_2}$}{
    $V_j= \emptyset$\;
    $E_j= \emptyset$\;
    \If{$j \leq d_{\rho_2}$}{
        $op = \rho_2$\;
    }
    \Else{
        $op = \rho_1$\;
    }
    \For{$i=1$ to $s$}{
        \If{$op = \rho_2$}{
            Sample a vertex $v_1$ from $I$ using a uniform distribution
            and sample a vertex $v_2$ from $V_{j-1}$ using the uniform distribution\;
            Create a vertex $v'$ such that $\phi(v')=(C',g)$ where $C'=\rho_2(C_1,C_2)$,
            $\phi(v_1)=(C_1,g)$ and $\phi(v_2)=(C_2,g)$\;
            $V_j=V_j \cup \{v'\}$\;
            $E_j=E_j\cup \{(v_1,v'),(v_2,v')\}$\;
        }
        \Else{
            Sample a vertex $v$ from $V_{j-1}$ using the uniform distribution\;
            Create a vertex $v'$ such that $\phi(v')=(C',g)$ where $C'=\rho_1(C)$ and $\phi(v)=(C,g)$\; 
            $V_j=V_j \cup \{v'\}$\;
            $E_j = E_j \cup\{(v,v')\}$\;
        }
    }
    $V = V \cup V_j$\;
    $E = E \cup E_j$\;
}
$O = \{v: v \in V, \nexists v' ~s.t.~ (v,v') \in E\}$\;
\Return $(V,I,O,E,\phi)$\;
\caption{\textbf{(RandomCRM)} Randomised construction of a CRM structure,
    with linear derivation of feature-clauses.}
\label{alg:randcrm}
\end{algorithm}

In the rest of the paper, we will use the term {\em Simple CRM\/}
 to denote a CRM constructed by either \autoref{alg:crm} or \autoref{alg:randcrm}
 in which the input clauses $C_1, C_2,\ldots, C_n$ are $\Mu$-simple feature-clauses.
 
 \subsubsection{Parameter Estimation}
 \label{sec:parest}
 
 \autoref{alg:randcrm} does not completely
 specify a CRM. Specifically, neither
 the edge-labelling $\psi$ nor $h$ are defined. We now
 describe a procedure that obtains a $\psi$ given
 the partial-specification returned by \autoref{alg:randcrm}
 and a pre-specified $h$
 suitable for the usual task of using the neural network for
 function approximation.  That is, given a partial
 specification of an unknown function
 $\delta:{\cal A} \to {\cal Y}$
 in the form of sample data
 $D = \{(a_i,y_i)\}_1^N$. We want the 
 the neural network to
 construct an approximation
 $\hat{\delta}:{\cal A} \to {\cal Y}$
 that is reasonably consistent with $D$.
In order to estimate the goodness of the approximation, we
need to define a {\em loss\/} function, that computes the penalty
of using $\hat{\delta}$. We will take $\hat{\delta}$ to be
synonymous with $h$, the computation function of the CRM.
Recall $h:{\mathbb{R}}^{|O|} \to \mathbb{R}^n$ for some fixed $n$.
In this paper, we will therefore take ${\cal Y} = \mathbb{R}^n$
and define $h$ in the usual manner adopted by neural networks,
namely as a function of ``local'' computations performed at 
each of the $O$ vertices of the CRM.

\begin{definition}[Local Computation in a CRM]
\label{def:local}
Let $(V,I,O,E,\phi,\psi)$
be a partially-specified CRM, where
$O = \{o_1,o_2,\ldots,o_k\}$. For each vertex $v_{i} \in V$  let
$\phi(v_i) = (C_i,g_i)$ and for each edge $(v_i,v_j)$ $\psi((v_i,v_j)) = w_{ij}$. 
Let $f_i$ denote $f_{C_i}$. For any
$a \in {\cal A}$ we define
$h_i:{\cal A} \rightarrow \mathbb{R}$ as follows:

\[ h_i(a)= \left\{
\begin{array}{ll}
        f_i(a)&\mbox{ if }v_i \in I \\
        f_i(a)g_i\left(\sum_{(v_j,v_i)\in E} w_{ji} h_j(a)\right)&\mbox{ if }v_i \not \in I \\
\end{array}
\right.
\]
Then $\hat{\delta}(a) = h(h_{o_1}(a),\ldots,h_{o_k}(a))$.
\end{definition}


For a multi-class classification task, function $h$ computes
the probability distribution over the classes,
for example, a $\mathtt{softmax}$ function.
Similarly, for a regression task, $h$ computes a real number,
for example, a $\mathtt{linear}$ function.

\autoref{alg:traincrm} estimates the parameters of
the neural network using a standard weight-update
procedure based on stochastic gradient descent (SGD)~\citep{rumelhart1986learning,goodfellow2016deep}, given the
structure obtained from \autoref{alg:randcrm},
a pre-defined computation function $h$, and a loss function $L$.
 
\begin{algorithm}[!htb]
\SetAlgoNoLine
\LinesNumbered
\KwData{A CRM structure $\gamma=(V,I,O,E,\phi)$, 
    a dataset $D = \{(a_i,y_i)\}_{1}^{N}$, where $a_i \in {\cal A}$ and $y_i \in \mathbb{R}^n$,
    a computation function $h: {\mathbb{R}}^{|O|} \to \mathbb{R}^n$,
    a loss function $L: \mathbb{R}^n \times \mathbb{R}^n \to \mathbb{R}$.
}
\KwResult{A CRM}
Let $O = \{o_1,\ldots,o_k\}$\;
Initialise $\psi$\;
\While{stopping criterion is not met}{
    Randomly draw an instance $(a_i,y_i)$ from $D$\;
    Let $\hat{y_i} = h(h_{o_1}(a),\ldots,h_{o_k}(a))$ (see \autoref{def:local})\;
    $Error = L(y_i,\hat{y_i})$\;
    Update $\psi$ using SGD to minimise $Error$\;
}
\Return $(V,I,O,E,\phi,\psi,h)$\;
\caption{\textbf{(TrainCRM)} Parameter estimation of a CRM, given
its structure, using stochastic gradient descent (SGD).
The training is done until some stopping criterion is reached, which refers to the condition when
the number of training epochs reaches some pre-specified
maximum value.
}
\label{alg:traincrm}
\end{algorithm}
 
\subsubsection{Predictions and Explanations}
\label{sec:expl}

We denote the prediction of a  CRM $\gamma = (V,I,O,E,\phi,\psi,h)$ 
by $\hat{\delta}(a)= h(h_{o_1}(a),\ldots,h_{o_k}(a))$,
where $O = \{o_1,\ldots,o_k\}$
and the $h_{o_i}$ are as defined in \autoref{def:local}.

The association of feature-clauses with every vertex of the CRM allows
us to construct ``explanations'' for predictions.
For this we introduce the notion of ancestral graph of a vertex and explanation graph of an output vertex in a CRM.

\begin{definition}[Ancestral Graph of a Vertex]
\label{def:anc}
Let $\gamma = (V,I,O,E,\phi,\psi,h)$ be a CRM.
The set of ancestors of a vertex $v \in V$ in $\gamma$, denoted by $Ancestors(v,\gamma)$, is defined as follows: 
\[
    Ancestors(v,\gamma) = 
        \left\{\begin{matrix}
        \{v\} & \mathrm{if}~ v \in I\\ 
     \bigcup_{(u,v) \in E} Ancestors(u,\gamma) \cup \{v\} & \mathrm{ otherwise}
        \end{matrix}\right.
\]

The ancestral graph of a vertex $v$ in $\gamma$ is $(V',E')$ where $V'=Ancestors(v,\gamma)$ and $E' = \{(u',u''): u', u'' \in V', (u',v'')$ in $E\}$. 

\end{definition}

\begin{definition}[Explanation Graph]
\label{def:expl}
Let $\gamma = (V,I,O,E,\phi,\psi,h)$ be a CRM, and $a \in {\cal A}$
be a data instance. Let $O = \{o_1,\ldots,o_k\}$
and let $\hat{\delta}(a) = h(h_{o_1}(a),\ldots,h_{o_k}(a))$
be the prediction of the
CRM for $a$. For $o_i \in O$, let
$C_i$ be the feature-clause associated with $o_i$
(that is, $\phi(o_i) = (C_i,\cdot)$). Let $f_i$ be
the corresponding feature-function (as defined in \autoref{sec:func}), and
$(V',E')$ be  an ancestral graph of $o_i$ in $\gamma$.
Then the explanation graph of $a$ from vertex $o_i$, denoted by ${Explain}_{\gamma,o_i}(a)$,  is as follows:
\[ Explain_{\gamma,o_i}(a)= \left\{\begin{matrix}
         (V',E',\phi') & \mathrm{if}~ f_i(a)=1\\ 
         \emptyset & \mathrm{ otherwise}
        \end{matrix}\right.
\]
where
$\phi': V' \rightarrow {\cal F}_\Mu$ is a vertex-labelling
function.
$\phi'(v) = C\theta_a$, where $\theta_a$
    the substitution $\{X/a\}$ for the variable $X$ in the head literal
    and $\phi(v) = (C,\cdot)$.
\end{definition}

\begin{remark}
${Explain}_{\gamma,o_i}(a)$ consists
of a (labelled) tree of feature-clauses extracted from
the derivation graph of feature-clauses. The root
of the tree is the feature-clause at $o_i$
and sub-trees contain simpler feature-clauses. If the
CRM is a Simple CRM, then the leaves
of the explanation-tree are $\Mu$-simple feature-clauses.
\end{remark}

\begin{example}
\label{ex:trainexpl}
In the Trains problem, suppose the data instance is the train
shown on the left below. The explanation graph, associated with an
output vertex of the CRM is shown on the right.

\smallskip
\begin{center}
\begin{tabular}{m{0.3\linewidth}lm{0.5\linewidth}} \\
\includegraphics[height=2.5cm]{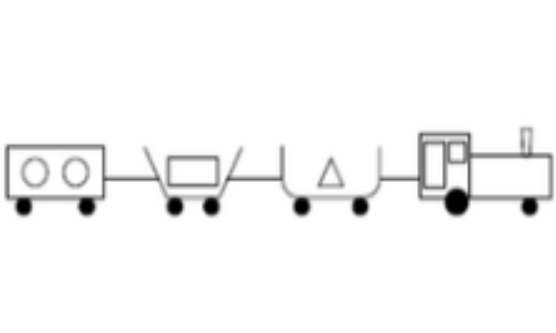} & \hspace*{0.5cm} &
\includegraphics[height=5cm]{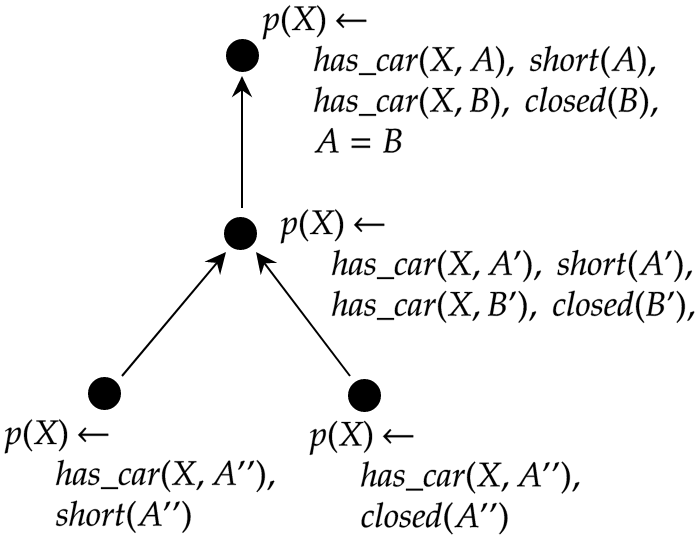} \\
(Train $t1$)       & & (additionally requires the substitution $\{X/t1\}$ to be applied)
\end{tabular}
\end{center}

By definition, we know
that the feature-function value associated with
$p(t1) \leftarrow has\_car(t1,A),$ $short(A),$
$has\_car(t1,B),$ $closed(B)$ has the value $1$.
Also, we know that feature-function values with all other
clauses in the explanation will also be $1$.
\end{example}

\smallskip
\noindent
The notion of an explanation graph from a vertex
extends naturally to the explanation graph from a set of vertices
which we do not describe here. It will be useful in what follows
to introduce the notion of
a feature-clause being ``contained in an explanation graph''.

\begin{definition}[Feature-clause Containment]
\label{def:contain}
Let $\gamma$ be a CRM, $o$ be an output vertex of $\gamma$,
Let $a$ be a data instance,
Let ${Explain}_{\gamma,o}(a)=(V,E,\alpha)$, and
$F_{\gamma,o}(a) =\{ \alpha(v) ~|~ v \in V\}$
be the set of feature-clauses in the explanation graph.
We will say a feature-clause $C$ is contained in
 ${Explain}_{\gamma,o}(a)$, or
$C \sqsubseteq {Explain}_{\gamma,o}(a)$ iff
there exists  $C' \in F_{\gamma,o}(a)$ s.t. $C\theta_a \equiv C'$
(where $\theta_a = \{X/a\}$, is a substitution for the variable
$X$ in the head of $C$).
\end{definition}

\noindent
(This naturally extends to the containment of a set of clauses.)

\begin{example}
\label{ex:contain}
The feature-clause $C:$ $p(X) \leftarrow$ $has\_car(X,Y),$ $short(X),$ $closed(X)$ is contained in the explanation graph shown for train $t_1$
in \autoref{fig:trainsexpl} because:

\begin{itemize}
    \item $F_{\gamma,o}(t_1)$ $=$ $\{C_1,C_2,C_3,C_4\}$ where:
            \begin{enumerate}
                \item[] $C_1:$ $p(t_1) \leftarrow (has\_car(t_1,A),$
                    $short(A),$ $has\_car(t_1,B),$
                    $closed(B),$  $A = B)$;
                \item[] $C_2:$ $p(t_1) \leftarrow (has\_car(t_1,A'),$
                    $short(A'),$ $has\_car(t_1,B'),$
                    $closed(B'))$;
                \item[] $C_3:$ $p(t_1) \leftarrow (has\_car(t_1,A''),$
                    $short(A''))$; and
                \item[] $C_4:$ $p(t_1) \leftarrow (has\_car(t_1,A''),$
                    $closed(A''))$
            \end{enumerate}
    \item  With $\theta_a = \{X/t_1\}$,  $C\theta_a$ = 
        $(p(t_1) \leftarrow (has\_car(t_1,A),$
                    $short(A),$ $closed(A))$; and
    \item $C\theta_a \equiv C_1$     
\end{itemize}
\end{example}

\subsection*{Explanatory Fidelity}
\label{sec:expfid}

Explanatory fidelity
refers to how closely the CRM's explanation matches the ``true
explanation''. Of course, in practice, explanatory fidelity
will be a purely notional concept, since the true explanation will
not be known beforehand. However it is useful for
us to calibrate the CRM's explanatory performance when
it is used for problems where true explanations are known
(the synthetic problems considered in experiments below are
in this category).

For a prediction $\hat{\delta}(h_{o_i}(a),\ldots,h_{o_k}(a))$
by a CRM, suppose we have a relevance ordering over
the output vertices $o_1,\ldots,o_k$. Let $o^*$ be
the most relevant vertex in this ordering. Then we will
call the explanation graph from $o^*$
as the {\em most-relevant explanation graph\/} for
$a$ given the CRM.\footnote{
In implementation terms, one way to obtain such a relevance
ordering over output vertices of the CRM is to use 
the $h_i(\cdot)$ values for vertices in
$O$ to select a vertex $o^*$ that has the highest
magnitude (this is the same as selecting
the best vertex after one
iteration of the layer-wise relevance propagation, or LRP~\citep{binder2016layer},
procedure).}

For a classification task, we use clause containment and
the most-relevant explanation graph
to arrive at a notion of explanatory
fidelity of a CRM to a set of feature-clauses ${\cal T}_c$ that
are known to be  `acceptable' feature-clauses for
class $c$ (if no such acceptable clauses exist for
class $c$, then ${\cal T}_c = \emptyset$).
Let $\gamma$ be a CRM used to predict the class-labels 
for a set of data-instances. For any instance $a$,
let $o^*$ denote the most relevant output vertex of the CRM.
We will say that a data instance $a$ is
consistently explained iff:
(i) the CRM predicts that $a$ has the class-label $c$; and
(ii) there exists a $C \in {\cal T}_c$ s.t.
$C \sqsubseteq {Explain}_{\gamma,o^*}(a)$; and
(iii) for $c' \neq c$, there does not exist $C' \in {\cal T}_{c'}$
    s.t. $C' \sqsubseteq {Explain}_{\gamma,o^*}(a)$.

Given a set of data-instances $E$, let
$CE$ denote the set of instances in $E$ explained consistently
and $IE$ denote the set of instances
in $E$ not explained consistently. Then the
explanatory fidelity of the CRM (correctly, this is only
definable w.r.t. the ${\cal T}_c$'s) is taken to be $\frac{|CE|}{|CE|+|IE|}$,
provided $(|CE| + |IE|) \neq 0$ (and undefined otherwise).

\subsection{CRMs as Explanation Machines}

CRMs can be used as `explanation machines' for black-box predictors
that do not intrinsically include an explanatory component. 
The approach, sometimes called 
{\em post hoc\/} explanation generation, is
shown in \autoref{fig:crmx}.

\begin{figure}
    \centering
    \includegraphics[width=0.7\linewidth]{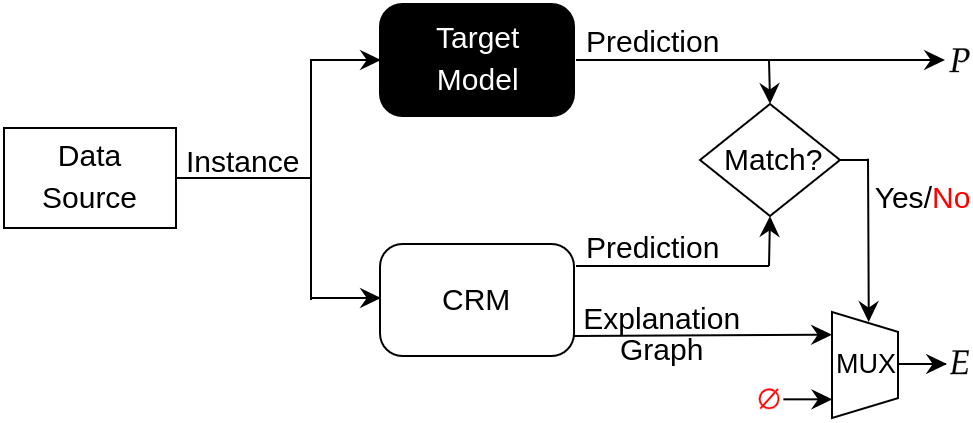}
    \caption{Using a CRM as an explanation machine.
        The target model is a ``black box'' that does not
        have an explanatory output. The CRM model is trained using
        training data labelled with the prediction from the black box
        (and not the `true label'). The multiplexer (MUX) selects 
        between the CRM's explanation graph and the ``empty'' explanation
        $\emptyset$ depending on whether the CRM's prediction
        does or does not match the target model's prediction.
        By using the setup as shown here, we are able to get a
        prediction $P$ and a corresponding explanation $E$.}
    \label{fig:crmx}
\end{figure}

To assess the utility of using the CRM in this manner, we will
change the usual assessment of predictive
accuracy to one of `predictive fidelity', which
refers to how closely the CRM matches the prediction
of the target model.

\section{Empirical Evaluation of CRMs as Explanation Machines}
\label{sec:expt}

\subsection{Aims}
\label{sec:aim}
We consider two kinds of experiments with Simple CRMs:
\begin{description}
    \item[Synthetic data.] Using tasks for which both target-model
        predictions and acceptable feature-clauses are available,
        we intend to investigate the hypothesis
        that:
        (a) Simple CRMs can construct models with high predictive
        fidelity to the target's prediction; and
        (b) Simple CRMs have high explanatory fidelity to the set of
        acceptable feature-clauses.
    \item[Real data.] Using real-world datasets, for which we have predictions
        from a state-of-the-art black-box target model, we investigate the
        hypothesis that Simple CRMs can construct models with
        high predictive fidelity to the target's prediction.
        We also provide illustrative
        examples of using the CRM to provide explanations for
        the predictions.
\end{description}

\noindent
We clarify what is meant by `acceptable feature-clauses' for the synthetic data
in \autoref{sec:methods}. For real data, 
the target-model is the
state-of-the-art (SOTA, which in this case is a graph-based neural network).
That is, the CRM 
is being used here to match the SOTA's predictions
(and not the `true value'), and to
provide proxy explanations. No acceptable feature-clauses are
known for classes in the real data.\footnote{
    The CRM can of course be used to predict the true value directly.
    We will comment on this later, but that
    is not the primary goal of the experiment here.}

\subsection{Materials}
\label{sec:mat}

\subsubsection{Data and Background Knowledge}

\begin{description}
    \item[Synthetic Data.] 
        We use two well-known synthetic datasets.
        The first dataset is the ``Trains'' problem of discriminating between
        eastbound (class = $+$) and westbound trains (class = $-$) \citep{michalski1980pattern}.
        The original data consists only of 10 instances (5 in each class). We generate
        a dataset of 1000 instances with a class-disitribution of approximately
        50\% $+$ and 50\% $-$, using the
        data generator \citep{michie1994international}. We
        use 700 instances  as training data and 300 instances as
        test-data.
        The second dataset consists of the task of discriminating between
        illegal (class = $+$) and legal (class = $-$) chess positions in the 
        King-Rook-King endgame \citep{bain1994learning,michie1976king}. The class-distribution is
        approximately 33\% $+$ and 67\% $-$. We use 10000 instances of board-positions
        as training data and 10000 instances as test-data. Examples of $+$
        instances are showed pictorially in \autoref{fig:expics}.
    
    \begin{figure}[htb]
        \begin{center}
            \begin{tabular}{ccc} \\
                \underline{Trains}& \hspace*{1cm} & \underline{Chess} \\ 
                \includegraphics[height=2.5cm]{train_example1.png} & & \includegraphics[height=3.5cm]{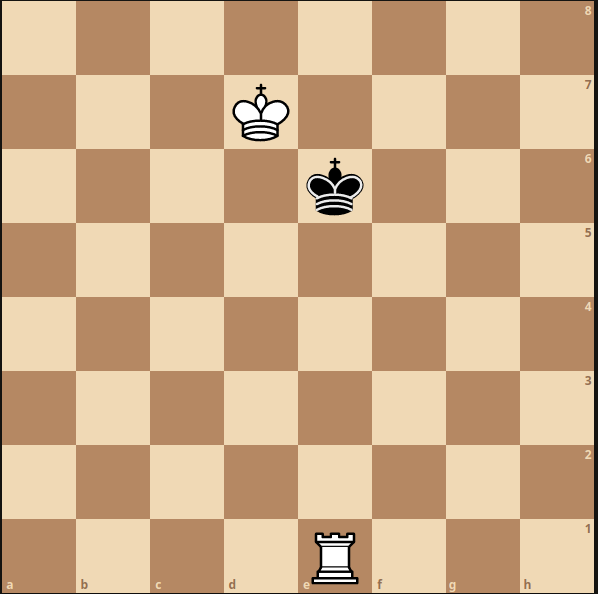} \\
            \end{tabular}
        \end{center}
    \caption{Pictorial examples of positive instances in the
    synthetic data. The actual data are logical encodings of
    examples like these. The instance on the left is
    an example of a train classified as ``eastbound'' ($+$).
    The instance on the right is of a board position
    classified
    as ``illegal'' ($+$), given that it is White's turn
    to move. For both problems, we have a target model
    that is complete and correct. We also know a set
    of feature-clauses that are acceptable as explanations
    for instances that are correctly predicted as 
    $+$.}
    \label{fig:expics}
    \end{figure}

    \item[Real Data.]
    Our real data consists of 10 datasets obtained 
    from the NCI\footnote{The National Cancer Institute (\url{https://www.cancer.gov/})}.
    Each dataset represents extensive drug evaluation with the concentration parameter GI50, 
    which is the concentration that results in 50\% growth inhibition of cancer cells~\citep{marx2003data}. 
    A summary of the dataset is presented in \autoref{fig:nci}.
    Each relational data-instance in a dataset describes a chemical compound
    (molecule) with atom-bond representation: a set of bond facts.
    The background knowledge consists of logic programs 
    defining almost $100$
    relations for various functional groups (such as amide, amine,
    ether, etc.) and various ring structures (such as aromatic, non-aromatic etc.).
    There are also higher-level domain-relations that determine
    presence of connected, fused structures.
    Some more details on the background knowledge can be seen in
    these recent studies:~\citep{dash2021incorporating,dash2022inclusion}.
    \begin{figure}[!htb]
        \centering
        \begin{tabular}{|c|c|c|c|c|}
        \hline
        \# of & Avg. \# of & Avg. \# of & Avg. \# of & \% of \\
        datasets & instances & atoms per instance & bonds per instance & positives \\
        \hline
        10 & 3018 & 24 & 51 & 50--75 \\
           &      &    &    & Avg.$=57$ \\
        \hline
    \end{tabular}
    \caption{Summary of the NCI-50 datasets (Total no. of instances is approx. 30,200). The graph neural network predictor described
    in \citep{dash2021incorporating} is taken as the target model.
    No acceptable feature-clauses are known for these tasks.}
    \label{fig:nci}
    \end{figure}
    
\end{description}

\subsubsection{Algorithms and Machines}

We use the ILP system Aleph~\citep{srinivasan2001aleph} for constructing the feature-clauses.
The CRMs are implemented using PyTorch~\citep{paszke2019pytorch}. The parameter
learning of CRMs has been done with the autograd engine
available within PyTorch for the implementation of backpropagation.
Our implementation of Layerwise-Relevance Propagation (LRP) is 
based on \citep{bach2015pixel,binder2016layer}. 

The CRM implementation and all our experiments are conducted 
on a workstation running with Ubuntu (Linux) operating system,
64GB main-memory, 
and a CPU running with 12 Intel Xeon processors.

\subsection{Method}
\label{sec:methods}

The experiments are in two parts: an investigation on synthetic
data to examine the predictive performance and explanatory
fidelity of CRMs; and an investigation on real data, to compare
the predictive performance of CRMs against state-of-the-art deep
networks. Some examples of explanations are also provided for the explanations
generated by a CRM on real data. We describe the method used for each part
in turn. 

\subsubsection{Experiments with Synthetic Data}

For both synthetic datasets, we have access to symbolic descriptions of
the true concepts involved. The `target model' in each case is taken
to be equivalent to a classifier that labels instances consistent with
the corresponding true concept. This allows us to
judge the fidelity of explanations generated.  The  method used in
each case is straightforward:

\begin{enumerate}
\item[] For each problem:
    \begin{enumerate}[label=(\alph*), leftmargin=1.5\parindent]
        \item Construct the dataset $D$ of instances labelled by the target model;
        \item Generate a subset of $\Mu$-simple feature-clauses in the
                mode-language for the problem;
        \item Randomly split $D$ into training and test samples;
        \item Construct a CRM using \autoref{alg:randcrm} with
            the $\Mu$-simple features. The weights
            for the CRM are obtained using the training data and
            the SGD-based weight-update steps described in
            \autoref{alg:traincrm}
            (see below for additional details);
        \item Obtain an estimate of the predictive and
            the explanatory fidelity of the CRM using the test
            data (again, see below for details).
    \end{enumerate}
\end{enumerate}

\noindent
The following additional details are relevant to the method
just described:

\begin{itemize}
        \item  For both datasets, the composition
        depth of CRMs is at most 3. Also, the mode-declarations for Chess allow the occurrence
        of equalities in $\Mu$-simple features
        (see \autoref{app:expt}), additional compositions using $\rho_1$
        are not used in this problem;
        \item We use the rectified linear ($\mathtt{relu}$) activation function for the 
        local computation in the neurons of the
        internal (hidden) layers of the CRMs. 
        \item We use Adam optimiser~\citep{kingma2015adam} to minimise the
            training cross-entropy loss between the true classes
            and the predicted classes by the network;
        \item We provide as input feature-clauses only a subset of all possible $\Mu$-simple
            feature-clauses. The subset is constrained by the
            following: (i) At most 2 body literals; (ii)
            Minimum support of at least 10 instances\footnote{
            In principle, increasing the number of input clauses will increase the size (width and breadth)
            of the CRM (measured by the number of layers and neurons in each layer of the CRM). Furthermore, since complex features are more specific than the features represented by simple clauses, the coverage of the complex features (that is, instances for which the features have the value 1) will usually be lower than those of simple clauses. Thus, if we restrict complex features
            to those having a positive coverage of at least $p$, simple features will have also have a coverage of at least $p$. Thus, simple features with
            a lower positive coverage will not be part of any connection in the CRM,
            and therefore need not appear in the inputs.}; 
            and
            (iii) Minimum precision of at least 0.5. All
            subsequent feature-clauses obtained by composition are also required to
            satisfy the same support and precision constraints.
            The learning rate for the
            Adam optimiser is set to $0.001$ while keeping other
            hyperparameters to their defaults within PyTorch; 
        \item The number of training epochs is set to $5$ for the Trains dataset 
            and $10$ for the Chess dataset; 
    \item Both synthetic problems are binary classification tasks.
        We call the classes $+$ and $-$ for convenience.
       Predictive fidelity is estimated in the usual manner,
        namely as the proportion of correctly predicted test-instances;
    \item Explanatory fidelity is estimated as described in \autoref{sec:expl}. 
        For this, we need to pre-define
        sets of feature-clauses that are acceptable in explanations.
        For the synthetic datasets, we are able to identify
        sets of acceptable feature-clauses from the literature:
        These feature-clauses are obtained from a target model
        that is known to be complete and correct
        (see \citep{michie1994international} for the target
        model for Trains and \citep{bain1994learning} for Chess).
        The acceptable clauses in ${\cal T}_+$ are as follows:
            \begin{center}
                {\small{
                \begin{tabular}{|c|l|} \hline
                    Problem & Acceptable Feature Clause\\ \hline
                    Trains  & $p(X) \leftarrow has\_car(X,Y), short(X), closed(X)$ \\ \hline
                    Chess   & $p((A,B,C,D,C,E)) \leftarrow true$ \\
                            & $p((A,B,C,D,E,D)) \leftarrow true$ \\
                            & $p((A,B,A,B,C,D)) \leftarrow true$ \\
                            & $p((A,B,C,D,E,F)) \leftarrow adj(A,E), adj(B,F)$ \\ \hline
                \end{tabular}
                }}
               
            \end{center}
                
    In Trains, the feature clauses apply to descriptions of trains.
        For Chess these are descriptions of the board in an endgame.
        The board is a  6-tuple that denotes the file and
        rank of the position of the White King, White Rook and Black King,
        respectively. In all cases ${\cal T}_{-} = \emptyset$. 
        Explanatory fidelity will be estimated by
        checking clause-containment of the feature-clauses above in
        the explanation graph for the most-relevant output vertex
        of the CRM (see \autoref{sec:expl});
    \item The acceptable feature-clause for the Trains is a direct
        rewrite of the function used to generate the labels.
        For Chess,
        the (set of) feature-clauses are direct rewrites of an approximate symbolic
        description taken from \citep{srinivasan1992distinguishing}. 
        These approximate description isn't 
        identical to the correct description, but is very closely related to it
        (the approximation differ from the correct description
        only in about 40 of 10,000 cases). For our purposes therefore,
        high explanatory fidelity, w.r.t. the set of feature-clauses shown,
        will taken to be sufficient; and
    \item We provide a baseline comparison for predictive
        fidelity against a `majority class' predictor.
        A baseline comparison is also provided for explanatory
        fidelity against random selection of a feature-clause
        from the set of feature-clauses associated with the output vertices of the CRM
        that have a feature-function
        value of 1 for the data instance being predicted.
\end{itemize}

\subsubsection{Experiments with Real Data}
\label{sec:realexpt}

For the real-world datasets, the current
state-of-the-art predictions are from
a Graph Neural Network (GNN)
constructed using the background knowledge described earlier \citep{dash2022inclusion}.
However, the GNN model constitutes a black-box
model, since it does not produce any explanations for its
predictions. We investigate equipping this black-box model
with CRM model for explanation. The method
is as follows:

\begin{enumerate}
\item[] For each problem:
    \begin{enumerate}[label=(\alph*), leftmargin=1.5\parindent]
        \item Construct the dataset $D$ consisting of
            problem instances and their predictions of the target model;
        \item Generate a subset of $\Mu$-simple feature-clauses in the
                mode-language for the problem. The restrictions
                used for synthetic data
                are used to constrain
                the subset;
        \item Construct a CRM using \autoref{alg:randcrm} with
            the $\Mu$-simple features and the dataset $D$. The weights
            for the CRM are obtained using the training data used
                by the state-of-the-art methods and
                the SGD-based weight-update procedure; and
        \item Obtain the predictive fidelity of the CRM model to
            the predictions of the target model.
    \end{enumerate}
\end{enumerate}

\noindent
The following additional details are relevant:
\begin{itemize}
    \item As with the synthetic data, the compositional
        depth for the CRMs is set to $3$. 
        Again, we do not use $\rho_1$ operations, since the mode-declarations
        allow equalities. The constraints on
        input feature-clauses is the same as those used
        for synthetic data;\footnote{Our choice for the compositional depth of $3$ is also loosely-based on our previous work on Deep Relational Machines (DRMs: \citep{dash2018large}).
        However, we note that a depth higher than $3$ will increase the complexity of a CRM
        (due to increase in number of layers and neurons), which        
        might result in better predictive performance of a CRM. We expect
        that in practice the depth bound will be treated as a hyperparameter and
        subject to the usual forms of hyperparameter optimisation.}
    \item The CRM implementation is the same as the one
        used for synthetic data. We
        perform a grid-search of the learning rate for the
        Adam optimiser using the parameter grid: $\{0.01, 0.001, 0.005, 0.0001\}$.
        The total number of training epochs is $10$,
        with early-stopping mechanism~\citep{prechelt1998early} with a patience period of $3$;
    \item As with the synthetic data, we provide a baseline comparison against the `majority
        class' predictor;
    \item Unlike the synthetic data, no pre-defined set of acceptable feature-clauses
        exists, and therefore no estimate of explanatory fidelity is
        possible. Correspondingly, there is no baseline provided either.
\end{itemize}

\subsection{Results}
\label{sec:results}

\autoref{fig:results} tabulates the results used to
compute estimates of predictive and explanatory fidelity on synthetic
and real data.
The main details
in these tabulations are these:
(a) For the synthetic data, Simple CRMs models 
    able to match the target's prediction perfectly
    (predictive fidelity of 1.0);
(b) The high explanatory fidelity values show that for instances
    labelled $+$, 
    the maximal explanation for the most-relevant
    vertex contains at least 1 clause from the set of acceptable feature-clauses;
    and for instances labelled $-$, the maximal explanation of the most relevant
    vertex does not contain any clauses from the target theory; and
(c) On the real datasets predictive fidelity of CRMs is reasonably high:
    suggesting that about 81\% of the time, the CRM's prediction will
    match that of the state-of-the-art model.

\begin{figure}[!htb]
    \begin{minipage}{0.5\textwidth}
    \centering
    \begin{tabular}{|c|c|c|c|c|}
        \hline
        Dataset & \multicolumn{4}{|c|}{Fidelity} \\ \cline{2-5}
                & \multicolumn{2}{|c|}{CRM} & \multicolumn{2}{|c|}{Baseline} \\ \cline{2-5}
                & Pred. & Expl. & Pred. & Expl. \\ \hline
        Trains  & 1.0 & 1.0 & 0.5 & 0.4 \\
        Chess   & 1.0 & 0.9 & 0.7 & 0.7\\
        \hline
    \end{tabular}
    \caption*{(a) Synthetic data}
    \end{minipage}
     \begin{minipage}{0.5\textwidth}
    \centering
    \begin{tabular}{|c|c|c|}
        \hline
        Dataset & \multicolumn{2}{|c|}{Pred. Fidelity} \\ \cline{2-3}
                &  CRM & Baseline \\ \hline
        786\_0	& 0.77 & 0.53 \\
        A498	& 0.79 & 0.59 \\
        A549\_ATCC & 0.85 & 0.63 \\
        ACHN	& 0.73 & 0.58 \\
        BT\_549	& 0.78 & 0.51 \\
        CAKI\_1	& 0.81 & 0.69 \\
        CCRF\_CEM & 0.82 & 0.68 \\
        COLO\_205 & 0.77 & 0.53 \\
        DLD\_1	& 0.90 & 1.00 \\
        DMS\_114 & 0.89 & 0.91 \\ \hline
	    Avg. & 0.81 (0.05) & 0.66 (0.17) \\
        \hline
    \end{tabular}
    \caption*{(b) Real data}
    \end{minipage}%
    \caption{Estimates of fidelity
        for Simple CRMs on the synthetic datasets and real datasets.
        Explanatory fidelity is assessable on synthetic data since
        we have access to the ``correct'' explanation. Baseline
        for prediction is the majority class predictor. For explanations,
        Baseline refers to random selection of a feature-clause from the set of feature-clauses associated with the output vertices of the CRM
        with function value $1$ for the data instance being predicted
        by the majority class predictor.}
    \label{fig:results}
\end{figure}

\noindent
We now turn to examine the results in greater detail.

\subsubsection{Predictive Fidelity}

Although we obtain perfect predictive fidelity to the target model
on synthetic data, fidelity on the real datasets clearly
has room for improvement. Improvements in fidelity
are possible simply by considering ensembles of CRMs,
obtained simply due to the sampling variation arising
in Step~(c) (refer \autoref{sec:realexpt}). Below, we tabulate changes in predictive fidelity 
on 1 of the real-world problems (786\_0), using a sample consisting
of upto 3 CRMs. With multiple CRMs, for a data-instance
to be correctly predicted it is sufficient
for any one of the CRMs to predict the same class as the target-model.
Recall the primary purpose of the CRM is to explain the target-model's prediction
in terms of its feature-clauses. Any CRM that matches the target-model's prediction
can be used to explain the prediction. More on this under ``Explanation''
below.

\begin{center}
  \begin{tabular}{|c|c|} \hline
        No. of & Predictive \\
        CRMs & Fidelity \\  \hline
        1 & 0.75 \\
        2 & 0.83 \\
        3 & 0.85  \\ \hline
    \end{tabular}
\end{center}

\subsubsection{Explanatory Fidelity} 

For the synthetic datasets
we show below in \autoref{fig:trainsexpl} a representative $+$ instance 
(shown pictorially for ease of
understanding), along with the 
predictions of both target and the CRM.
The last column shows an acceptable feature-clause along 
with a stylised English translation. In both instances,
an equivalent form of the acceptable feature-clause is
contained in the CRM's explanation graph.

\begin{figure}[!htb]
\centering
\begin{sideways}
    \begin{tabular}{m{0.23\linewidth}m{0.5\linewidth}m{0.45\linewidth} } 
    \hline
    \multicolumn{1}{c}{Instance} & \multicolumn{1}{c}{Explanation Graph} & \multicolumn{1}{c}{Acceptable Feature Clause} \\
    \hline
    \includegraphics[width=\linewidth]{train_example1.png} & 
    \includegraphics[width=0.68\linewidth]{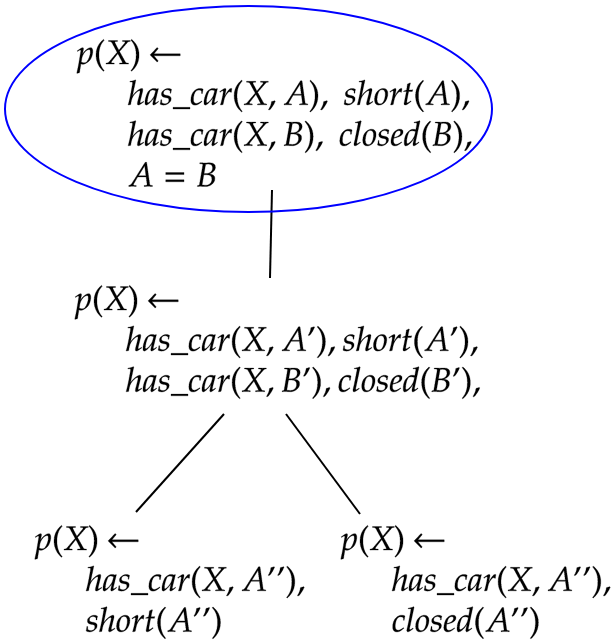} & 
    $p(X) \leftarrow has\_car(X,Y), short(Y), closed(Y)$ \\
     & & Train $X$ has a car $Y$ and $Y$ is short and closed. \\
    Train $t_1$ & (With the substitution $\{X/t1\}$) & \\

    \hline
    & & \\
    \includegraphics[width=\linewidth]{chess_example1.png} &
    \includegraphics[width=\linewidth]{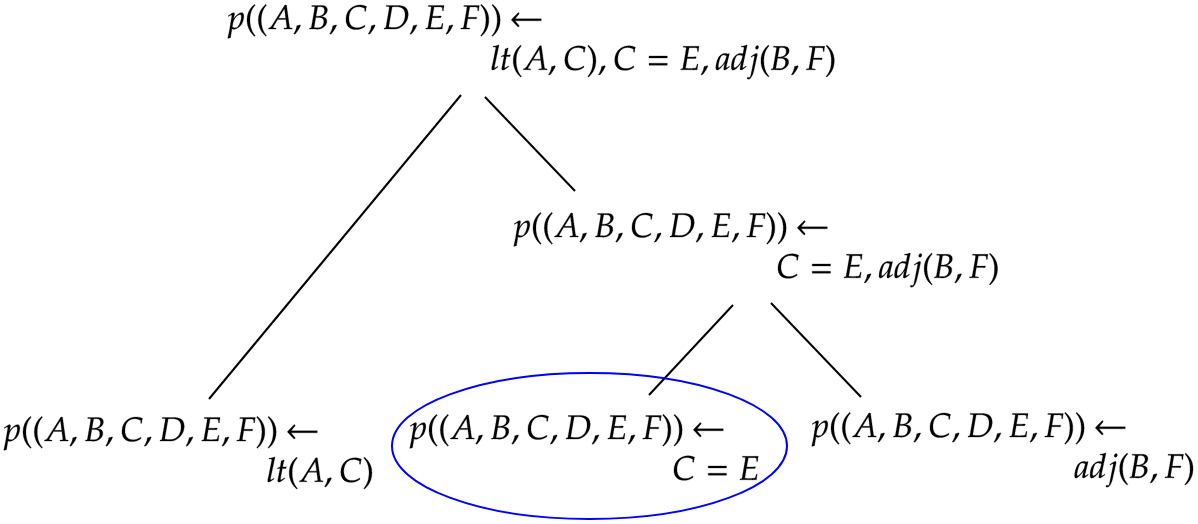} &  $p((A,B,C,D,C,E)) \leftarrow $ \\
    & & White Rook and Black King are on the same file (column) \\
        Board $(d,7,e,1,e,6)$ & (With the substitution $\{A/d, B/7, \ldots, F/6\}$) & \\

    \hline
    \end{tabular}
\end{sideways}
\caption{Explanations from the CRMs for a $+$ prediction by both target-model and
the CRM for two instances in the synthetic data. In both
cases, the CRM's explanation graph contains an acceptable feature-clause (circled).}
\label{fig:trainsexpl}
\end{figure}

For the Chess data, the CRM's explanatory fidelity is less than $1$.
This means that there are instances for which the CRM's explanation graph does
not contain an acceptable feature-clause. We discover 19 different kinds of
`buggy' explanations are found by the CRM: a full listing
is in \autoref{app:expt}. Here
we provide illustrative instances of two kinds of errors:
those that are close to the correct explanation; and those that
are an artifact of the specific data instance being
explained (see \autoref{fig:chesswrong}). Besides these, in many cases, we find
explanation errors arise from the fact that definitions in background knowledge
of file and rank adjacency hold when ranks and files are the same
(that is, if $A$ and $B$ are
ranks (or files) and $A = B$, then $adj(A,B)$ is true).\footnote{M. Bain,
the author of the background definitions, confirms
that this is the intended meaning of the $adj/2$ predicate for
this problem (personal communication).} In many instances
inconsistent explanations result from feature-clauses contain literals that
use equality instead of adjacency (that is, the CRM's explanation
contains $A=E$, rather than $adj(A,E)$: see \autoref{app:debug} in \autoref{app:expt}).
However, even accounting for this,
the CRM's explanations can be more specific than the correct explanation;
and in some cases, incorrect (an example of each is shown in \autoref{fig:chesswrong}).

\begin{figure}[!htb]
{\small{
\begin{center}
    \begin{tabular}{m{0.2\linewidth}m{0.45\linewidth}m{0.25\linewidth} } 
    \hline
    \multicolumn{1}{c}{Instance} & \multicolumn{1}{c}{Explanation} & \multicolumn{1}{c}{Acceptable} \\
     & \multicolumn{1}{c}{Graph}& \multicolumn{1}{c}{Feature-Clause} \\ 
    \hline
    & & \\
    \includegraphics[width=\linewidth]{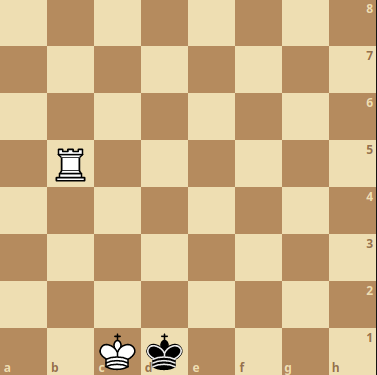} & 
    \includegraphics[width=\linewidth]{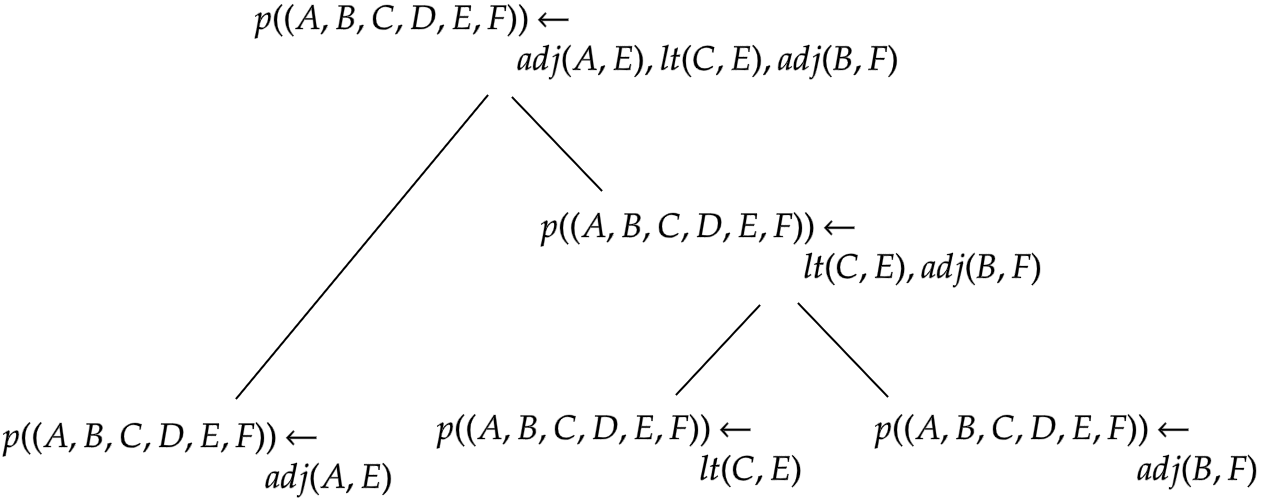} & \\
    & & $p((A,B,C,D,E,F)) \leftarrow \textcolor{white}{aaa} adj(A,E), adj(B,F)$ \\
    \includegraphics[width=\linewidth]{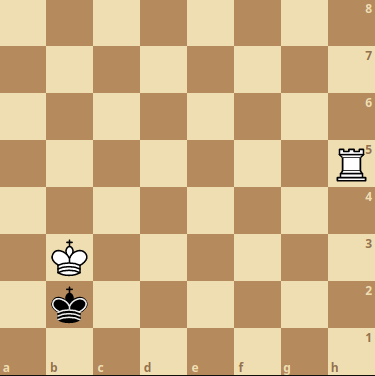} & 
    \includegraphics[width=\linewidth]{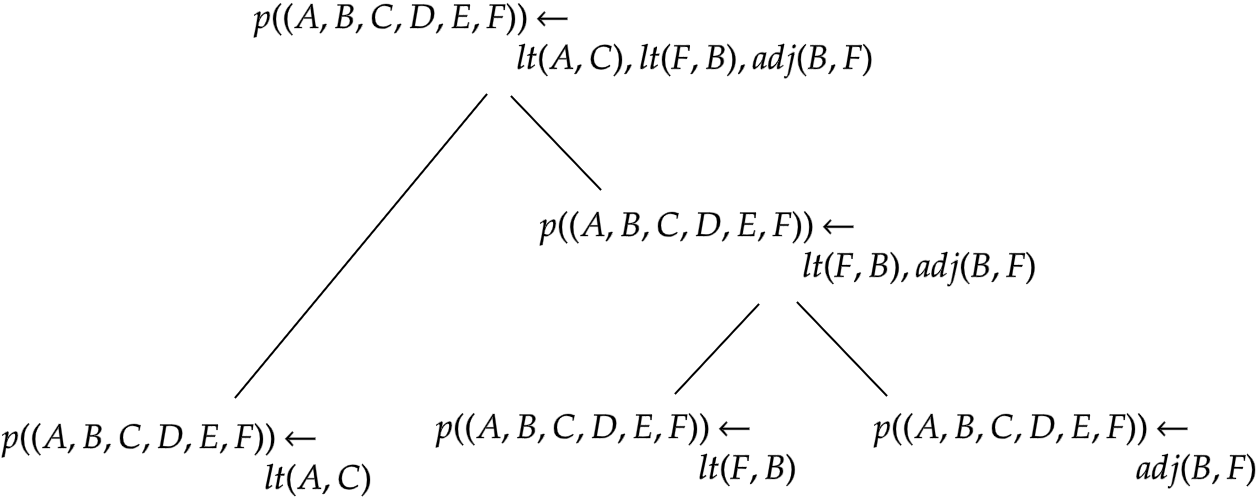} & \\
    & & White King's file is adjacent to Black King's file and White King's rank is adjacent to Black King's rank \\
    \hline
    \end{tabular}
\end{center}
}}
\caption{Examples where the target model and the CRM's prediction are both $+$, but
the CRM's explanation graph does not contain an
acceptable feature-clause. For simplicity, we do not show the
substitutions for $A \ldots F$.}
\label{fig:chesswrong}
\end{figure}


What about explanations on real data? At this point, we do not have any
independent source of acceptable feature-clauses for this data.
We nevertheless show a representative example of the explanation for
a test-instance (see \autoref{fig:realexpl}).

\begin{figure}[!htb]
{\small{
\begin{center}
    \begin{tabular}{m{0.33\linewidth}m{0.56\linewidth} } 
    \hline
    \multicolumn{1}{c}{Instance} & \multicolumn{1}{c}{Explanation Graph} \\
    \hline
    \includegraphics[height=2cm]{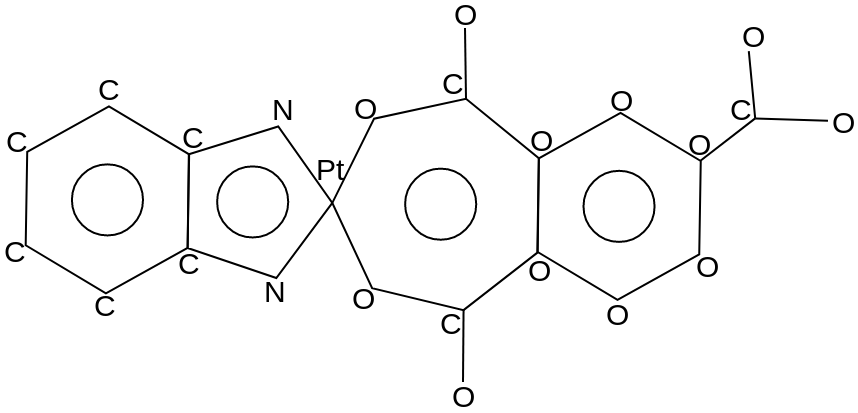} &
    \includegraphics[height=4.7cm]{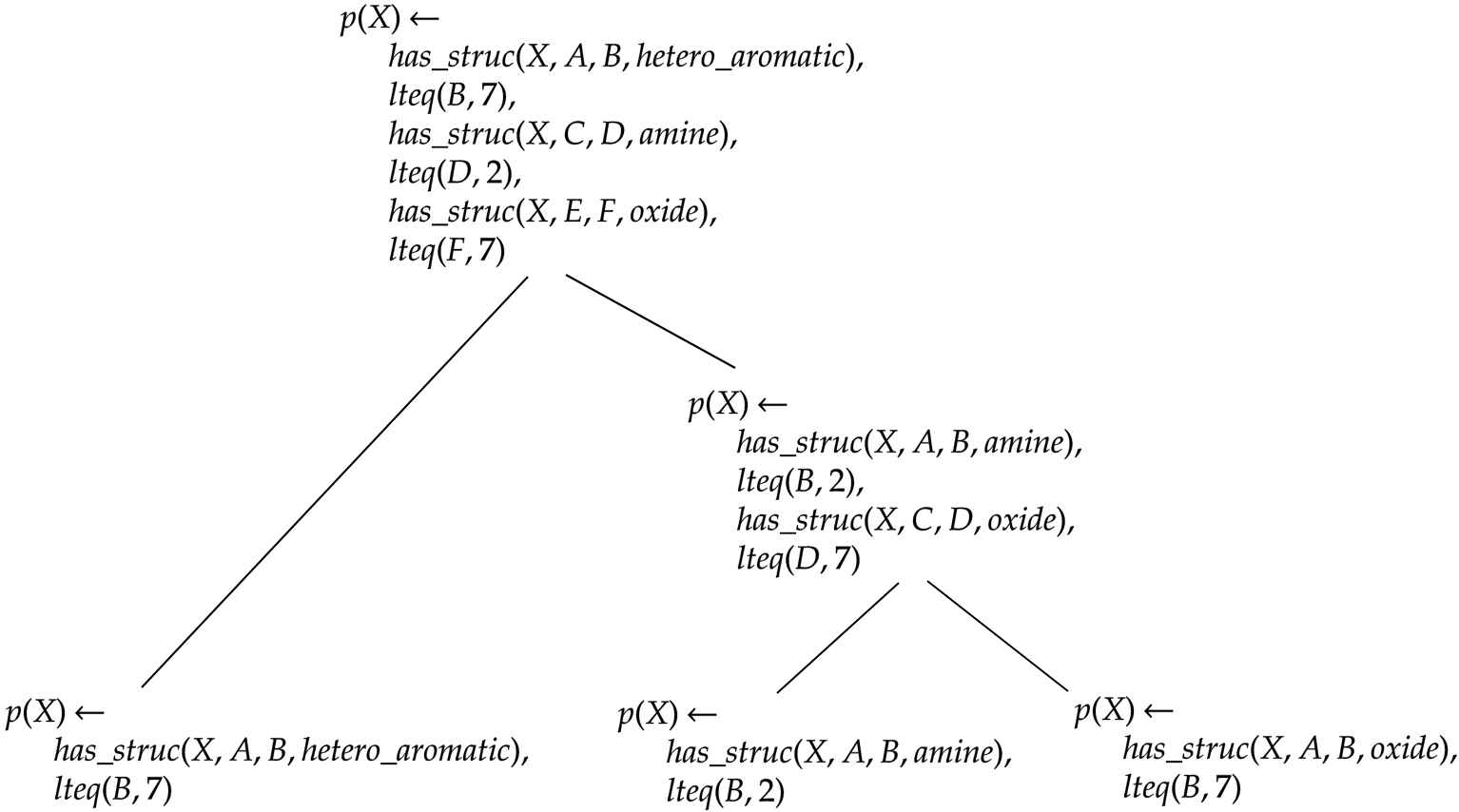} \\
    \hline
    \end{tabular}
\end{center}
}}
\caption{A CRM explanation for a prediction on the real data.
The class-label predicted by both CRM and the black-box model for this data instance is $+$.}
\label{fig:realexpl}
\end{figure}


We close this examination by drawing the reader's attention
to an important aspect of a CRM's
explanation. The feature-clauses are defined in terms
of relations provided as prior knowledge. This makes
them potentially intelligible
to a person familiar with the meanings of these relations.
This makes it easier--in principle at least--to perform
a human-based assessment of the feature-clauses in the
explanation graph (this is apparent from the `debugging' of
explanations that we have been to accomplish with the Chess data).

\subsubsection{Additional Results: CRMs as Prediction Machines}

The tabulations of fidelity and the subsequent assessments above provide
a measure of confidence in the use of CRMs as explanation machines. But it
is evident that CRMs can be used as `white-box' predictors in their own
right. We provide an indicative comparison of a CRM predictor against the
state-of-the-art predictors (for the real-data, the prediction is
by majority-vote from an ensemble of 3 CRMs):

\begin{figure}[!htb]
    \centering
   \begin{tabular}{|c|c|c|c|c|c|}
    \hline
    Dataset & \multicolumn{5}{c|}{Predictive accuracy} \\ \cline{2-6}
               & CRM         & GNN         & DRM (500)   & CILP++      & Baseline \\
    \hline 
    786\_0     & 0.66 (0.01) & 0.69 (0.01) & 0.69 (0.01) & 0.67 (0.01) & 0.55 (0.01) \\ 
    A498       & 0.67 (0.01) & 0.72 (0.01) & 0.70 (0.01) & 0.66 (0.01) & 0.52 (0.01) \\
    A549\_ATCC & 0.64 (0.01) & 0.67 (0.01) & 0.70 (0.01) & 0.60 (0.01) & 0.51 (0.01) \\
    ACHN       & 0.64 (0.01) & 0.70 (0.01) & 0.70 (0.01) & 0.64 (0.01) & 0.51 (0.01) \\
    BT\_549    & 0.66 (0.01) & 0.68 (0.01) & 0.70 (0.01) & 0.65 (0.01) & 0.53 (0.01) \\
    CAKI\_1    & 0.63 (0.01) & 0.68 (0.01) & 0.66 (0.01) & 0.64 (0.01) & 0.54 (0.01) \\
    CCRF\_CEM  & 0.65 (0.01) & 0.71 (0.01) & 0.71 (0.01) & 0.68 (0.01) & 0.63 (0.01) \\
    COLO\_205  & 0.60 (0.01) & 0.69 (0.01) & 0.67 (0.01) & 0.66 (0.01) & 0.56 (0.01) \\
    DLD\_1     & 0.69 (0.02) & 0.69 (0.02) & 0.70 (0.02) & 0.72 (0.02) & 0.69 (0.02) \\
    DMS\_114   & 0.68 (0.02) & 0.74 (0.02) & 0.75 (0.02) & 0.75 (0.02) & 0.76 (0.02) \\
    \hline
    \end{tabular}
    \caption{Predictive performance comparison of an ensemble of $3$ Simple CRMs against some
        leading black-box predictors. The numbers are estimates of predictive accuracy
        obtained on test data. The number in parentheses is the estimated standard
        deviation. The techniques being compared  are:
        GNN, the graph-based neural network approach described in \citep{dash2022inclusion};
        DRM ($500$), is a form of MLP called a Deep Relational Machine that has, as input, Boolean
            feature-vectors resulting from a stochastic selection of $500$ relational
            features \citep{dash2019discrete}; and CILP++, an MLP that has input
            Boolean feature-vectors resulting from an exhaustive feature
            construction technique called Bottom-Clause Propositionalisation \citep{francca2014fast}.
            Baseline is the majority class predictor.}
    \label{fig:predresults}
\end{figure}

The results in \autoref{fig:predresults} indicate that Simple CRMs perform
approximately as well as CILP++~\citep{francca2014fast}, 
but are worse than either the GNN~\citep{dash2022inclusion} or DRM~\citep{dash2019discrete}.
However, \autoref{fig:predresults} are best treated as preliminary.
Variations in CRMs arise
in \autoref{alg:randcrm} purely due to sampling, of course.
However, the CRM obtained is
also affected by the following: (1) The subset
constraints on support and precision
all feature-clauses in the CRM; (2) bounds on the
depth of compositions $\rho_2$ followed by $\rho_1$ operators; 
(3) The number of feature-clauses drawn in each layer of the
CRM. Additional variation can arise from the initialisation
of weights for the SGD-based estimation of parameters. This suggests
that substantially more experimentation is needed to see if the predictive performance
of Simple CRMs can be improved. We note also 
that the DRM uses substantially more complex features than the Simple CRM,
and that CILP++ constructs substantially more features than
the Simple CRM (anywhere between 30,000 to 50,000 compared to about 330
$\Mu$-simple features for the CRMs). Of course none of GNN, DRM or
CILP++ have any intrinsic mechanism of associating explanations with
their prediction. 

\section{Related Work}
\label{sec:relworks}

We note first that $\rho_1$ and $\rho_2$ are closely related to the notion of refinement
operators which have been studied extensively in ILP, in the context
of the search through a hypothesis space (see \citep{tamaddoni2009lattice,shanhwiei:ilpbook}).
Our motivation in this paper is, however, in the use of these operators
to derive relational features. Consequently, we 
describe connections to related work in 3 categories:
conceptual work on relational features; implementation and
applied work on propositionalisation in ILP; and
work on explainable deep networks.

On the conceptual understanding of relational features,
perhaps the most relevant single summary is in \citep{saha2012kinds}. There the
authors identify several feature-classes, based on
somewhat similar notions of source- and sink-literals. The relationship between
the different classes in that paper is shown in \autoref{fig:fclasses}(a). The relationship
to these sets of the class of $\Mu$-simple feature-clauses,
denoted here as $F_{\Mu}$,
is shown in \autoref{fig:fclasses}(b) (see \autoref{app:fclasses}
for more details).

\begin{figure}[htb]
\centering
\includegraphics[width=0.6\linewidth]{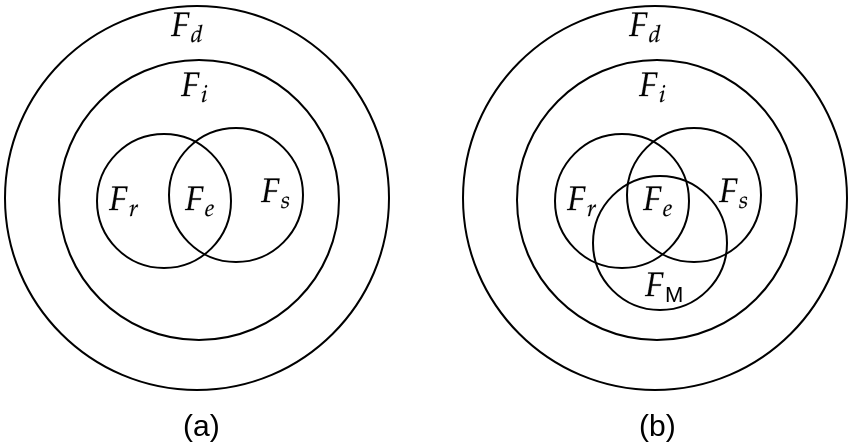}
\caption{The feature classes proposed in \protect{\citep{saha2012kinds}}: (a)
    and the relationship to the class of $\Mu$-simple features (b).}
\label{fig:fclasses}
\end{figure}

 Simple clauses in \citep{eric:thesis} and the corresponding set
 of features in $F_s$ are restricted to determinate
 predicate-definitions.\footnote{
  Informally, a {\em determinate\/} predicate is one whose definition encodes a function. That
  is, for a given set of values for input arguments, there is exactly one set of values for
  output arguments.}
  Results in \citep{eric:thesis} show that features from $F_s$
  can be used to derive the subset of feature-clauses in $F_d$
  that only contain determinate predicate-definitions.
  There is no such restriction imposed on $F_{\Mu}$ and
  all clauses in $F_d$ (and therefore all other classes shown)
  can be derived using some composition of $\rho_1$ and $\rho_2$. No corresponding
  operators or completeness results are known for $F_r$.
  
  Relational features have been shown to be an extremely
  effective form of learning with relational objects \citep{Kramer2001,lavravc2021propositionalization}.
  Methods that construct and use relational features, guided by some form of mode-declarations
  can be found in \citep{srinivasan1999feature,lavravc2002rsd,ramakrishnan2007using,joshi2008feature,specia2009investigation,faruquie2012topic,saha2012kinds,francca2014fast,vig2017investi,dash2018large}. 
  Of these, the features in
  \cite{lavravc2002rsd} are from the feature class $F_r$. There are no 
  reports on the class of features used in the other reports, although
  the procedures for obtaining the features suggest that they are
  not restricted to any specific sub-class (that is, they are simply
  from $F_d$). Given our results on derivation of features in $F_d$
  from features in $F_{\Mu}$, and the class-inclusions shown, we
  would expect at least some features in a super-class would
  require additional composition operations to those in a sub-class.
  In terms of a CRM structure, we would expect features
  in $F_i$, for example, would usually be associated with vertices at a greater
  depth than those in $F_r$.
   Empirical results tabulated for some statistical learners in \citep{saha2012kinds}
  suggest that relational features from the class $F_i$ were most useful
  for statistical learners.  If this empirical trend continues
  to hold, then we would expect the performance of CRMs to improve as
  depth increases (and features in $F_i$ are derived), and then to flatten
  or decrease (as features in $F_d \setminus F_i$ are derived).
  
 The development and application of CRMs is most closely related to the
 area of {\em self-explainable\/} deep neural networks \citep{alvarez2018towards,angelov2020towards,ras2022explainable}.
 The structure of the CRM enforces a meaning to each node in the network,
 and in turn, we have shown here how to extract one form of explanation from these
 meanings. A different kind of neural network, also with meanings associated
 nodes is described in \citep{sourek2018lifted}. Those networks are also explainable,
 although not in the manner described here. 
 In \citep{srinivasan2019logical}, 
 a symbolic proxy-explainer is constructed using ILP for a form of multi-layer perceptron
 (MLP) that uses as input values of relational
 features. The features there drawn from the class $F_d$, and
 the explanations are logical rules constructed by ILP using the
 feature-definitions provided to the MLP. There are at least
 two important differences to the explanations constructed there and the ones
 obtained with a CRM: (i) The rules constructed in \citep{srinivasan2019logical}
 effectively only perform the $\rho_2$ operation on relational
 features. This can result in a form of incompleteness: some
 features in $F_d$ cannot be represented by the rules, unless
 they are already included as input to the MLP; and
 (ii) The structuring of explanations in \citep{srinivasan2019logical} requires relevance
 information: here, the structuring is from usual functional
 (de)composition.

\section{Conclusion}
\label{sec:concl}

It has been long-understood in 
machine learning that the choice of representation can make
a significant difference to the performance and efficiency
of a machine learning technique. Representation is also clearly
of relevance when we are interested in constructing human-understandable
explanations for predictions made by the machine learning technique. A
form of machine learning that has paid particular attention to issues
of representation is the area of Inductive Logic Programming
(ILP). A form of representation that has been of special interest
in ILP is that of a {\em relational feature\/}. These are Boolean-valued
functions defined over objects, using definitions of relations
provided as prior- or background knowledge. The use of relational
features forms the basis of an extremely effective form of ILP called
{\em propositionalisation\/}. This obtains a Boolean-vector description
of objects in the data, using the definition of the relational features
and the background knowledge. Despite the obvious successes of
propositionalisation, surprising little is known, conceptually, about the 
space of relational features. In this paper, we have sought
to address this by examining relational features within a mode-language
$\Mu$, introduced in ILP within the setting of mode-directed
inverse entailment \citep{muggleton1995inverse}. Within a mode-language,
we identify the notion of $\Mu$-simple relational
features, and two operations $\rho_1$ and $\rho_2$ that
allows us to compose progressively
more complex relational feature in the mode-language. In
the first half of the paper, we show that $\rho_1$ and $\rho_2$
are sufficient to derive all relational features within a
mode-language $\Mu$. This generalises a previous result
due to \cite{mccreath1998lime}, which was restricted
to determinate definitions for predicates in the background knowledge, albeit starting from a different definition of
simple features to that work.

In the second half of the paper, we use the notion of $\Mu$-simple features
and the composition operators $\rho_1$ and $\rho_2$ to construct a kind
of deep neural network called a Compositional Relational
Machine, or CRM. A special aspect of CRMs is that we
are able to associate a relational feature with each node
in the network. The structure of the CRM allows us to identify further
how the feature at the node progressively decomposes
into simpler features, until an underlying set of $\Mu$-simple features
are reached. This corresponds well to the intuitive notion of
a structured explanation, that is composed of increasingly simpler
components. We show how this aspect of CRMs allows them to be used
as ``explanation machines'' for black-box models. Our results on
synthetic and real-data suggest that CRMs can reproduce target-predictions
with high fidelity; and the explanations constructed on synthetic data
suggest that CRM's explanatory structure usually also
contains an acceptable explanation.

We have not explored the
power of CRMs as ``white-box'' predictors in their own right, but
early results suggest that it may be possible to obtain
CRMs with good predictive accuracy. Although still significantly
lower than the state-of-the-art, we believe this can
change. We have also not
explored other forms of CRMs, both simpler and more elaborate. 
For example, the identification
 of $\Mu$-simple features and their subsequent compositions using the
 $\rho$-operators suggests an even simpler CRM structures than that
 used here. It is possible for example, simply to obtain all possible
 compositions to some depth, and use a Winnow-like parameter estimation \citep{littlestone1988learning} to obtain a self-explainable linear model. Equally, more complex CRMs
 are possible by incorporating weights on the $\Mu$-simple features
 (this could be implemented simply by changing the activation function
at the input nodes of the network). Taking this one step further, it
is possible to associate weights with all the relational features, which
will allow the use of the inference machinery of probabilistic logic
programs~\citep{de2019neuro}.
We think an investigation of these other kinds of compositional relational machines
would contribute positively to the growing body of work in human-intelligible
machine learning.

\begin{acknowledgements}
AS is a Visiting Professor at Macquarie University, Sydney
and a Visiting Professorial Fellow at UNSW, Sydney. He is also the Class of 1981
Chair Professor at BITS Pilani, Goa and a Research Associate at TCS Research.
AS and TD would like to thank Lovekesh Vig and Gautam Shroff at TCS Research for
interesting discussions on explainable neural networks; and Michael Bain at UNSW
for discussions on the use of ILP for constructing symbolic explanations.
\end{acknowledgements}

\section*{Declaration}

\begin{description}
\item[Funding] Not applicable.
\item[Conflicts of interest] Not applicable.
\item[Ethics approval] Not applicable.
\item[Consent to participate] Not applicable.
\item[Consent for publication] Not applicable.
\item[Data and code availability] All data and codes used in our research can be found at: \url{https://github.com/tirtharajdash/CRM}.
\item[Authors' contributions] 
AS and AB conceived and worked on the conceptual parts related to simple features and their composition, and the specification of CRMs. TD conceived and worked on the implementation of CRMs as gated neural networks. AS and TD conceived and worked on the application of CRMs to synthetic- and real-data. DS was involved in some parts of the implementation of CRMs.
\end{description}

\clearpage

\appendixpage

\begin{appendices}

\section{Logic Terminology}
\label{app:logic}

In this section we cover only terminology used in the paper,
and further confined largely to logic programming. 
For additional background and further terminology
see~\citep{lloyd2012foundations,chang2014symbolic,nilsson1991logic,muggleton1994inductive}. 
The summary below is adapted from \citep{srinivasan2019logical}.

A language of first order logic programs has a vocabulary of constants,
variables, function symbols, predicate symbols, logical
implication `$\leftarrow$', and punctuation symbols.
A function or predicate can have a number of arguments known as
\emph{terms}.
Terms are defined recursively.
A constant symbol (or simply ``constant'') is a term.
A variable symbol (or simply ``variable'') is a term.
If $f$ is an $m$-ary function symbol, and $t_1, \ldots, t_m$ are terms,
then the function $f(t_1, \ldots, t_m)$ is a term.
A term is said to be \emph{ground} if it contains no variables.

We use the convention used in logic programming when writing
clauses.  Thus, predicate, function and constant symbols are written as
a lower-case letter followed by a string of lower- or upper-case letters,
digits or underscores (`\_').
Variables are written similarly, except that the first letter must be upper-case.
This is different to the usual logical notation, where predicate-symbols start with
upper-case, and variables start with lower-case: however the logic programming
syntax is useful for the implementation of CRMs.
Usually, predicate symbols will be denoted by symbols like $p,q,r$, etc.,
and symbols like $X,Y,Z$ to denote variables.
If $p$ is an  $n$-ary predicate symbol, and $t_1, \ldots, t_n$ are terms,
then the predicate $p(t_1, \ldots, t_n)$ is an atom.
Predicates with the same predicate symbol but different arities are
distinguished by the notation $p/n$ where $p$ is a predicate of arity $n$.

A literal is either an atom or the negation of an atom.
If a literal is an atom it is referred to as a positive literal, otherwise
it is a negative literal.
A clause is a disjunction of the form
$A_1 \vee \ldots \vee A_i \vee \neg A_{i+1} \vee \ldots \vee \neg A_k$,
where each $A_j$ is an atom.
Alternatively, such a clause may be represented as an implication (or ``rule'')
$A_1, \ldots,  A_i \leftarrow A_{i+1}, \ldots,  A_k$.
A definite clause $A_1 \leftarrow A_2, \ldots,  A_k$ has exactly one
positive literal, called the \emph{head} of the clause, with the
literals $ A_2, \ldots,  A_k$ known as the \emph{body} of the clause.
A definite clause with a single literal is called a \emph{unit} clause,
and a clause with at most one positive literal is called a Horn clause.
A set of Horn clauses is referred to as a logic program.
It is often useful to represent a clause as a set of literals.

A \emph{substitution} $\theta$ is a finite set $\{v_1/t_1, \ldots, v_n/t_n\}$
mapping a set of $n$ distinct variables $v_i$, $1 \leq i \leq n$, to terms
$t_j$, $1 \leq j \leq n$ such that no term is identical to any of the variables.
A substitution containing only ground terms is a \emph{ground} substitution.
For substitution $\theta$ and clause $C$ the expression $C\theta$ denotes
the clause where every occurrence $C$ of a variable from $\theta$ is
replaced by the corresponding term from $\theta$.
If $\theta$ is a ground substitution then $C\theta$ is called a ground
clause.
Since a clause is a set, for two clauses $C$, $D$, the set inclusion
$C\theta \subseteq D$ is a partial order called \emph{subsumption},
usually written $C$ $\theta$-subsumes $D$ and denoted by $C \preceq D$.
For a set of clauses $S$ and the subsumption ordering $\preceq$, we have
that for every pair of clauses $C, D \in S$, there is a least upper bound
and greatest lower bound, called, respectively, the least general
generalisation (lgg) and most general unifier (mgu) of $C$ and $D$, which
are unique up to variable renaming.
The subsumption partial ordering on clauses enables the definition of a lattice,
called the \emph{subsumption lattice}.

We assume the logic contains axioms allowing for inference using the equality predicate $=/2$.
This includes axioms for reflexivity ($\forall x (x = x)$), and substitution
($\forall x ~(\phi(x) \wedge (x = y) \longrightarrow \phi(y))$).

\section{Mode Language}
\label{app:modes}

We borrow some of the following definitions from \citep{dash2022inclusion}. The definition of $\lambda\mu$ sequence is simplified as we are dealing only with feature clauses here and all other definitions are same as in \citep{dash2022inclusion}.

\begin{definition}[Term Place-Numbering]
Let $\pi = \langle i_1,\ldots, i_k \rangle$
be a sequence of natural numbers.
We say that a term $\tau$ is in
place-number $\pi$ of a literal $\lambda$ iff: 
(1) $\pi \neq \langle\rangle$; and
(2) $\tau$ is the term at place-number $\langle i_2, \ldots, i_k \rangle$ in
the term at the $i_1^{\mathrm {th}}$ argument of $\lambda$.
$\tau$ is at a place-number $\pi$ in term $\tau'$: 
(1) if $\pi = \langle\rangle$ then
$\tau=\tau'$; and
(2) if $\pi = \langle i_1,\ldots,i_k\rangle$
then $\tau'$ is a term of the form $f(t_1,\ldots,t_m)$, 
$i_1 \leq m$ and $\tau$ is in place-number
$\langle i_2,\ldots,i_k\rangle$ in $t_{i_1}$. 
\label{def:placenum}
\end{definition}

\begin{definition}[Type-Names and Type-Definitions]
\label{def:types}
Let $\Gamma$ be a set of types and
$\Tau$ be a set of ground-terms. For 
$\gamma \in \Gamma$ we define a set of
ground-terms $T_\gamma$ =
$\{\tau_1,\tau_2,\ldots\}$, where $\tau_i \in \Tau$.
We will say a ground-term
$\tau_i$ is of type $\gamma$ if
$\tau_i \in T_\gamma$, and
denote by $T_\Gamma$ the set
$\{T_\gamma: \gamma \in \Gamma\}$.
$T_\Gamma$ will be called a set of type-definitions.
\end{definition}
\noindent

\begin{definition}[Mode-Declaration]
\label{def:modedecs}

\begin{enumerate}[(a)]
    \item Let $\Gamma$ be a set of type names. A mode-term  is defined recursively
    as one of:
    (i) $+\gamma$, $-\gamma$ or
    $\#\gamma$ for some $\gamma \in \Gamma$; or
    (ii) $\phi({mt}_1',{mt}_2',\ldots,{mt}_j')$, where
        $\phi$ is a function symbol of arity $j$,
        and the ${mt}_k'$s are mode-terms.
        
\item A mode-declaration $\mu$ is of the form
    $modeh(\lambda')$ or $modeb(\lambda')$.
    Here $\lambda'$ is a ground-literal of
    the form $p({mt}_1,{mt}_2,\ldots,{mt}_n)$ where
    $p$ is a predicate name with arity $n$, and
    the ${mt}_i$ are mode-terms. We will
    say $\mu$ is a $modeh$-declaration
    (resp. $modeb$-declaration)
    for the predicate-symbol $p/n$. In general
    there can be several $modeh$ or $modeb$-declarations for
    a predicate-symbol $p/n$.
    We will use $ModeLit(\mu)$ to denote $\lambda'$.
\item $\mu$ is said to be a mode-declaration for a literal
    $\lambda$ iff $\lambda$ and $ModeLit(\mu)$
    have the same predicate
    symbol and arity.
\item Let $\tau$ be the term at place-number $\pi$ in $\mu$,
    We define
    \[ModeType(\mu,\pi) = \left\{\begin{matrix}
    (+, \gamma) & ~\mathrm{if}~ \tau= +\gamma \\
    (-, \gamma) & ~\mathrm{if}~ \tau= -\gamma \\
    (\#, \gamma) & ~\mathrm{if}~ \tau= \#\gamma \\
    unknown      & ~\mathrm{otherwise}
    \end{matrix}\right.\]
\item If $\mu$ is a mode-declaration for literal $\lambda$,
    $ModeType(\mu,\pi)$ = $(+,\gamma)$
    for some place-number $\pi$, 
    $\tau$ is the term at place $\pi$ in
    $\lambda$, 
    then we will say $\tau$ is an input-term
    of type $\gamma$ in $\lambda$ given $\mu$
    (or simply $\tau$ is an input-term of type
    $\gamma$). Similarly we define
    output-terms and constant-terms.
\end{enumerate}

We will also say that mode $\mu$ contains an input argument of type $\gamma$ if there exists
    some term-place $\pi$ of $\mu$ s.t. $ModeType(\mu,\pi) = (+,\gamma)$. Similarly for output arguments.
    
\end{definition}

\begin{definition}[$\lambda\mu$-Sequence]
\label{def:lmpair}
Assume a set of type-definitions
$T_\Gamma$, modes $\Mu$.
Let $\langle C \rangle = \langle l_1,\neg l_{2},l_3\ldots,\neg l_k \rangle$ be an
ordered clause.
Then 
$\langle (\lambda_1,\mu_1),(\lambda_2,\mu_2),\dots, (\lambda_k,\mu_k) \rangle$
is said to be a $\lambda\mu$-sequence for $\langle C \rangle$ iff it
satisfies the following constraints:

\begin{description}
\item[Match.]
    (i) $\lambda_i = l_i$;  
    (ii) For $j=1,\dots,k$, $\mu_j$ is a mode-declaration for $\lambda_j$ s.t.
        $\mu_j = modeh(\cdot)$ ($j=1$)  and $\mu_j = modeb(\cdot)$ ($j > 1$).
\item[Terms.] 
(i) If $\tau$ is an input- or output-term in $\lambda_j$ given $\mu_j$,
                then $\tau$ is a variable in $\lambda_j$; 
(ii) Otherwise if $\tau$ is a constant-term in
                $\lambda_j$ given $\mu_j$ then $\tau$ is a ground term.
\item[Types.]
    (i) If there is a variable $v$ in both $\lambda_i$, $\lambda_j$ then
        the type of $v$ in $\lambda_i$ given $\mu_i$ is the same as the type of $v$ in $\lambda_j$ given $\mu_j$;
    (ii) If $\tau$ is a constant-term in $\lambda_i$ and the type of $\tau$ in $\lambda_i$ given
        $\mu_i$ is $\gamma$, then $\tau \in T_\gamma$.
\item[Ordering.]
    (i) If $\tau$ is an input-term in $\lambda_j$ given $\mu_j$ and $j > 1$ then there is an input-term $\tau$ in $\lambda_1$     given     $\mu_1$; or there is an output-term $\tau$ in $\lambda_i$ ($m < i < j$) given $\mu_i$.
    (ii) If $\tau$ is an output-term in 
        $\lambda_1$ given $\mu_1$, then $\tau$ is an output-term of some $\lambda_i$ ($1 < i \leq k$)
        given $\mu_i$.
\end{description}

\end{definition}

\begin{definition}[Mode-Language]
\label{def:modelang}
Assume a set of type-definitions
$\Tau$ and  modes  $\Mu$. The mode-language ${\cal L}_{\Mu,\Tau}$  
is $\{\langle C \rangle:$ 
either $C=\emptyset$  or  
there exists a $\lambda\mu$-sequence for $\langle C\rangle \}$.
\end{definition}

\noindent



\section{Proof of the Derivation Lemma}
\label{app:deriv}

\begin{lemmap}{\ref{lemma:deriv}}[Derivation Lemma]
Given $\Mu, \Mu'$ and $\Omega = \{\rho_1,\rho_2\}$ as before. Let
$\langle C \rangle$ be an ordered clause in ${\cal L}_{\Mu,\Tau}$, with
head $p(X)$.
Let $S$ be a set of ordered $\Mu$-simple clauses in ${\cal L}_{\Mu,\Tau}$, with
heads $p(X)$ and all other variables of clauses in $S$ standardised apart from each other and from $C$. If there
exists a substitution $\theta$ s.t.
$Basis(\langle C \rangle)\subseteq S\theta $ then
there exists an ordered clause $\langle C' \rangle$
in ${\cal L}_{\Mu',\Tau}$  such that $\langle C' \rangle$ is equivalent to $\langle C \rangle$ and derivable from $S$ using $\Omega$.
\end{lemmap}

\begin{proof}
We prove this in 3 parts:

\begin{enumerate}
    \item We first show: if $Basis(\langle C \rangle)\subseteq S\theta $,
        then $\theta$ is a type-consistent substitution for clauses in $S$. 
        \begin{itemize}
            \item[] Let $Basis(\langle C \rangle)$ =
            $\{\langle C_1 \rangle, \ldots, \langle C_k \rangle\}$
            and $S' = \{\langle S_1 \rangle, \ldots, \langle S_k \rangle\} \subseteq S$.
            Without loss of generality, let $\langle S_i \rangle \theta = \langle C_i \rangle$.
            Let $\theta = \{Y_1/t_1, Y_2/t_2,\ldots,Y_n/t_n\}$
                where the $Y_i$ are variables in $S$ and the $t_i$ are
                terms in $Basis(\langle C \rangle)$. In general,
                the $t_i$'s are variables, constants, or functional terms.
                Let $t_j$ be a constant. Suppose clause $S_i$ has
                the variable $Y_j$ in some location in $S_i$ and clause $C_i$ has the
                constant $t_j$ in the corresponding location. Since both
                $S_i$ and $C_i$ are in ${\cal L}_{\Mu,\Tau}$, this is only
                possible if there are multiple mode-declarations for the
                corresponding literals in $S_i$ and $C_i$.
                But constraint MC2 (in \autoref{sec:featclauseinmode}) ensures that there is exactly one mode-declaration
                for any literal. Therefore $t_i$ cannot be a constant. Reasoning
                similarly, $t_i$ cannot be a functional term. Therefore $t_i$
                has to be a variable in $C_i$.
                Since $Y_j$ and $t_j$ are in the same locations for
                some literal $l$ in $S_i$ and $C_i$, and there is
                only one mode-declaration for $l$ in $\Mu$, it follows
                that the types of $Y_j$ and $t_j$ must be the same.
        \end{itemize}
        
    \item If there exists a type-consistent substitution $\theta$ for
        clauses in $S$, then there is an ordered clause $\langle C'' \rangle$
        in ${Closure}_{\{\rho_2\}}(S)$
        s.t. $C'' \theta \equiv C$.
        \begin{itemize}
            \item[]  Let us fix an ordering among the simple clauses in $S$:
                $S_1,\ldots,S_k$, where $S_i=p(X) \leftarrow 
                Body_i(X,\mathbf{Y_i})$. Let us define a sequence of ordered clauses 
                $\langle C_1\rangle$, $\langle C_2 \rangle$, $\dots$, $\langle 
                C_{k}\rangle$ where $\langle C_1 \rangle$ is $\langle S_1 \rangle$, and
                $\langle C_i \rangle$ is in $\rho_2(\langle C_{i-1}\rangle,
                \langle S_{i} \rangle)$ for $i=2$ to $k$. It 
                is easy to see that the above sequence is a derivation 
                sequence of $\langle C_{k} \rangle$ from $\{\langle S_1 
                \rangle, \langle S_2 \rangle, \dots, \langle S_k \rangle\}$
                using $\{\rho_2\}$. That is,
                $\langle C_{k} \rangle$ is in ${Closure}_{\{\rho_2\}}(S)$.
                $\langle C_{k} \rangle$ is of the form $p(X) \leftarrow 
                Body_1(X,\mathbf{Y_1}), Body_2(X,\mathbf{Y_2}),$ $\dots, 
                Body_k(X,\mathbf{Y_k})$.
                That is
                $C_k = \bigcup_{i=1}^k S_i$.
                Since $S\theta = Basis(\langle C \rangle)$, then
                $C = \bigcup_{i=1}^k S_i\theta$.
                Let $\langle C'' \rangle$ = $\langle C_k \rangle$
                That is, $C'' \theta = C$. Hence
                $C'' \theta \equiv C$.
                Also, since $\theta$ is a type-consistent substitution for clauses
                    $S_i \in S$,  and $C'' = \bigcup_i S_i$, $\theta$
                    is a type-consistent substitution for $C''$.
            \end{itemize}
        \item   From above $\langle C'' \rangle $ and $\theta$,
            we construct an ordered clause
            $\langle C' \rangle$ in ${\cal L}_{\Mu',\Tau}$ s.t.
            $\langle C' \rangle$ is in $Closure_{\{\rho_1\}}(\{\langle C'' \rangle\})$ and
            $C'' \theta \equiv C'$.
            \begin{itemize}
                \item[]             
                Let $\langle C'' \rangle$ $= (p(X) \leftarrow Body(X,{\mathbf Y}))$,
                and $\theta = \{Y_1/Y_1', 
                \ldots, Y_l/Y_l'\}$. It is easy to see that $\theta$ 
                induces an equivalence relation on the variables in 
                $C''$. Variables $Y_i$ and $Y_j$ are in the same 
                equivalence class if both variables map to the 
                same variable $Y_m'$. For each equivalence class 
                $[U]$, we earmark one element in the class as 
                representative of the class. This element is denoted by $rep([U])$.
                Let $C''$ be $p(X) \leftarrow Body(X,{\mathbf Y})$.
                Let us fix an order among the variables occurring in 
                $\langle C'' \rangle$: $Y_1,Y_2,\dots,Y_l$. 
                Consider the following sequence of ordered clauses $\langle 
                C_0\rangle, \langle C_1 \rangle, \dots, \langle C_l 
                \rangle$ where $\langle C_0\rangle =\langle 
                C''\rangle$, and
                $\langle C_i \rangle = 
                p(X) \leftarrow Body(X,\mathbf{Y}), Y_1=rep([Y_1])$, $Y_2=rep([Y_2]), \dots, Y_l=rep([Y_l])$. 
                The above sequence is a derivation of $\langle C_l 
                \rangle$ from $\langle C'' \rangle$ using only $\rho_1$.
                That is
                $\langle C_l \rangle$ is in ${Closure}_{\{\rho_1\}}(\{\langle C'' \rangle \})$.
                Let $\langle C' \rangle$ be $\langle C_l 
                \rangle$ and it is of the form
                $p(X) \leftarrow Body(X,{\mathbf Y}), Y_1 = rep([Y_1]),\ldots, Y_l = rep([Y_l])$.
                It is easy to see that $\langle C' \rangle \in {\cal L}_{\Mu',\Tau}$.
                We now show $C''\theta  \equiv C'$. Let
                $\theta = \theta'' \circ \theta'$,
                where $\theta'' = \{Y_1/rep([Y_1]), \ldots, Y_l/rep([Y_l]\}$
                and $\theta' = \{rep([Y_1])/Y_1', \ldots,$ $
                rep([Y_l])/Y_l'\}$.\footnote{Correctly,
                $\theta'$ and $\theta''$ have to be functions.
                It is evident that $\theta''$ is a function, since
                each variable maps to the representative of its equivalence
                class. We argue informally that $\theta'$ is a function as
                follows. If $Y_i$ and $Y_j$ are 
                in the same equivalence class,  then $rep([Y_i])$ and 
                $rep([Y_j])$  are the same variables. But according to our 
                $\theta'$, they are mapped to $Y_i'$ and $Y_j'$. One might
                think $\theta'$ is not a function. But the 
                variables in the same equivalence class are mapped to same 
                variable. So $Y_i'$ is the same as $Y_j'$ and hence $
                \theta'$ is a function.}
                The variables $rep([Y_i])$ are in $S$ and the variables 
                $Y_i'$ are in $C$. Since the variables in $S$ are 
                standardised apart from the variables in $C$, $\theta'$ is
                a renaming substitution.      
                Therefore
                $C'' \theta' \equiv C''$.
                Therefore $C'' \theta$ $\equiv$  $C''\theta''$.
                By the substitution axiom in the equality logic
                $C''\theta''$ $\equiv$ $C'$. Therefore $C''\theta \equiv C'$.
        \end{itemize}
\end{enumerate}

 \noindent
From (1)--(3) above, the result follows. \qed
\end{proof}

\section{Properties of Derivations using \texorpdfstring{$\rho_1,\rho_2$}{rho}}
\label{app:linderiv}

\begin{lemma}[Reordering Lemma]
Given $M,M'$ and a set
of ordered clauses $\Phi$. If  an ordered clause $\langle C \rangle$ is derivable from $\Phi$ using $\{\rho_1,\rho_2\}$
then:
(a) there is a derivation for $\langle C_j \rangle$ from $\Phi$ using $\{\rho_
2\}$;
(b) there is a derivation for $\langle C_k \rangle$ from $\{\langle C_j \rangle\}$ using $\{\rho_1\}$; and
(c) $\langle C \rangle$ is equivalent to $\langle C_k \rangle$.
\end{lemma}

\begin{proof}
Let us assume there is a derivation $\tau$ for $\langle C \rangle$ from $\Phi$ using $\{\rho_1,\rho_2\}$. There exists an ordered clause $\langle C_k \rangle$ equivalent to $\langle C \rangle = p(X) \leftarrow l_1,l_2,\dots, l_j,l_{j+1},l_{j+2},\dots,l_k$ such that \begin{itemize}
\item the sequence of literals $l_1,\dots, l_j$ can be split into $Body(X,Y_1),\dots, Body(X,Y_m)$ such that $p(X) \leftarrow Body(X,Y_i) \in \Phi$ for each $i=1,\dots,m$, and
\item $l_{j+1},\dots, l_k$ are equality predicates introduced using $\rho_1$ operator in the derivation $\tau$. Let $\langle C_j \rangle$ be $p(X) \leftarrow 
 Body(X,Y_1),\dots, Body(X,Y_m)$. 
\end{itemize}
 It is easy to see that there is derivation $\tau_1$ for $\langle C_j \rangle$  from $\Phi$ using $\rho_2$ operator. Now we can apply $\rho_1$ operator on $\langle C_j \rangle$ repeatedly to derive $\langle C_k \rangle$.  \qed
\end{proof}

\begin{lemma}
Given $M,M'$ and a set
of ordered clauses $\Phi$. If an ordered clause $\langle C \rangle$ is derivable from $\Phi$ using $\{\rho_2\}$
then it is linearly derivable from $\Phi$ using $\{\rho_2\}$.
\end{lemma}
\begin{proof}
Let us assume that $\langle C \rangle$ is derivable from $\Phi$ using $\{\rho_2\}$. Hence there exists 
a derivation sequence $\xi$ of ordered clauses $\langle C_1 \rangle, \langle C_2 \rangle, \dots, \langle C_n \rangle$ from $\Phi$ using $\{\rho_2\}$ and $\langle C_n \rangle =\langle C \rangle$. Now we use mathematical induction to prove that for each $1\leq i \leq n$, there exists a linear derivation of $\langle C_i \rangle$ from $\Phi$ using $\{\rho_2\}$. 

\begin{description}
    \item[Base step:] Case $n=1$ is easy as  $\langle C_1 \rangle$ should be in $\Phi$ and hence there is a linear derivation for $\langle C_1 \rangle$ from $\Phi$ using $\{\rho_2\}$. 
    \item[Induction step:] We assume there
    exists a linear derivation for $\langle C_i \rangle$ from $\Phi$ using $\{\rho_2\}$ for each $i<j$ (induction hypothesis) and prove that there  exists a linear derivation for $\langle C_j \rangle$. The clause $\langle C_j\rangle$ occurs in the derivation sequence $\xi$. There are two cases to be considered: it is in $\Phi$ or it is in $\rho_2(\langle C_l\rangle, \langle C_m\rangle)$ where $l,m <j$. If $\langle C_j \rangle$ is in $\Phi$, there exists a linear derivation sequence for $\langle C_j \rangle$ from $\Phi$ (just the one step derivation containing itself). Now we prove the claim for the second case. Since  $l,m < j$, by induction hypothesis, there 
    exist linear derivation sequences $\xi_l$ for $\langle C_l 
    \rangle$ and $\xi_m$ for $\langle C_m \rangle$ from $\Phi$ using $\{\rho_2\}$. Let $\langle C_1'\rangle, \langle C_2' \rangle,\dots, \langle C_k'\rangle $ be the clauses in $\Phi$ and occurring in $\xi_m$. Now
    consider the following sequence $\xi_j$ containing $\xi_l$, $\langle C_1'\rangle$, $\langle C_1'' \rangle$,
    $\langle C_2'\rangle$, $\langle C_2''\rangle$,\dots, $\langle C_k' \rangle$,$\langle C_k'' \rangle$ where 
    $\langle C_1'' \rangle \in \rho_2(\langle C_l \rangle, 
    \langle C_1' \rangle)$ and
    $\langle C_a'' \rangle \in \rho_2(\langle C_{a-1}'' \rangle, 
    \langle C_a' \rangle)$ for $a=2$ to $k$. It is easy to see that $\xi_j$ is a
    linear derivation sequence for $\langle C_j \rangle$ from $\Phi$ using $\rho_2$. Hence there
    exists a linear derivation sequence for $\langle
    C_n \rangle$ from $\Phi$ using $\{\rho_2\}$. 
\end{description}
\qed
\end{proof}

Now the linear derivation lemma (re-stated below) can be easily proved by combining the above two lemmas. 

\begin{lemmap}{\ref{lemma:linderiv}}[Linear Derivation Lemma]
Given $M,M'$ and a set
of ordered clauses $\Phi$. If an ordered clause $\langle C \rangle$ is derivable from $\Phi$ using $\{\rho_1,\rho_2\}$
then there exists an equivalent ordered clause $\langle C' \rangle$ and it is linearly derivable from $\Phi$ using $\{\rho_1,\rho_2\}$.
\end{lemmap}

\section{Relationship of \texorpdfstring{$\Mu$}{M}-Simple Feature-Clauses to Known Feature Classes}
\label{app:fclasses}

We note the following about the sets of feature-clauses
in \autoref{fig:fclasses}(b) and the set $F_{\Mu}$ of $\Mu$-simple feature clauses:

\begin{description}

    \item[The set $F_d$.] This is identical to the set of ordered feature-clauses in the mode-language $\Mu$;
    
    \item[The set $F_i$.] This set is defined based on a given mode language $M$ (but without constraints 
    MC--MC3). An ordered feature clause  $\langle C \rangle$ is in $F_i$ iff the number of connected 
    components after removing the source vertex (corresponding to the head literal) from the clause dependency graph 
    of $\langle C \rangle$ is exactly one. Since the clause dependency graph of a clause in $F_\Mu$ has
    exactly one sink literal and any vertex in that graph has a path to that sink vertex (\autoref{rem:singlesource}).  So the number of connected components after removing the source vertex is one. Hence $F_\Mu \subseteq F_i$. It is easy to see that $F_i \not \subseteq F_\Mu$ (see the counterexample given for $F_s \subseteq F_\Mu$). 
    
    \item[The set $F_s$.] This set is defined based on the 
        class of {\em simple\/} clauses identified by  \cite{eric:thesis}
        who proposed {\em simple clauses\/}. We note:
        \begin{enumerate}
            \item Feature-clause definitions in $F_s$ do not refer to a mode-language. The
                clause dependency-graph is constructed using a procedure described in \citep{eric:thesis},
                and is based on the re-occurrence of variables (without any reference to input
                or output variables or types). $\Mu$-simple feature-clauses require
                a mode-language, with the constraints MC1--MC3;
            \item $F_{\Mu} \not\subseteq F_s$.
                For example assume a mode language $p(+int), q(+int,-int), r(+int)$ and the
                feature-clause $p(X) \leftarrow q(X,X),r(X)$. The clause dependency-graph of this clause has only one sink vertex and so it is in $F_{\Mu}$. But this is not in $F_{s}$
                since the directed graph associated with this clause has two sink literals. 
        \item $F_{s} \not\subseteq F_{\Mu}$. 
            For example assume a mode language $p(+int)$, $q(+int,-int)$, $r(+int,-int)$ and the
            feature-clause $p(X) \leftarrow q(X,Y),r(X,Y)$. This is in $F_s$,
            but is not in $F_{\Mu}$. The cause dependency-graph described in
            this paper doesn't have an edge between the vertices for $q$ and $r$ resulting in two
            sink vertices. But the dependency graph constructed in \citep{eric:thesis} will have an edge between the vertices for $q$ and $r$ because of the shared variable $y$.
    \end{enumerate} 
    \item[The set $F_r$.] The set $F_r$ consist of feature-clauses designed for
        {\em subgroup discovery} in relational data \citep{lavravc2002rsd}. Then:
            \begin{enumerate}
                \item Feature-clauses in $F_r$ do require a mode-language, and we can
                    construct a clause dependency-graph as described
                    here. The clause dependency-graph for feature-clauses in $F_r$
                    have exactly one component, and all new existential 
                    variables introduced by a source literal appear in source or sink literals.                 
                \item $F_{\Mu} \not\subseteq F_r$. For example assume a mode language 
                    $p(+int), q(+int,-int)$ and the clause $p(X) \leftarrow q(X,Y)$. 
                    This is a $\Mu$-simple clause but not in $F_r$ as the existential variable 
                    introduced at the source literal $q$ but it is not appearing later.
                \item $F_{r} \not \subseteq F_{\Mu}$. For example assume a mode language 
                    $p(+train)$, $has\_car(+train, -car)$, $short(+car)$, $closed(+car)$ and the 
                    clause $p(X) \leftarrow has\_car(X,Y), short(Y),closed(Y)$. This is not 
                    $\Mu$-simple clause there are two sink literals in it but it is in $F_r$. 
            \end{enumerate}
\end{description}

\section{Additional Details Relevant to the Experiments}
\label{app:expt}

\subsection{Examples of Mode Declarations and Simple Feature-Clauses} 

\autoref{fig:trainmodes}--\autoref{fig:ncimodes} show examples of mode declarations
used for the experiments. Also shown are examples of $\Mu$-simple feature-clauses
constructed automatically from the mode-declarations.

\begin{figure}[!htb]
    \centering
    \begin{tabular}{ll} 
    \hline
    \multicolumn{1}{c}{Modes} & \multicolumn{1}{c}{$\Mu$-simple feature clauses}\\ 
    \hline
    {\tt modeh(p(+tr))} & $p(A) \leftarrow has\_car(A,B)$ \\ 
    {\tt modeb(short(+car))} & $p(A) \leftarrow has\_car(A,B),short(B)$ \\ 
    {\tt modeb(closed(+car))} & $p(A) \leftarrow has\_car(A,B),closed(B)$ \\ 
    {\tt modeb(long(+car))} & $p(A) \leftarrow has\_car(A,B),long(B)$ \\ 
    {\tt modeb(open\_car(+car))} & $p(A) \leftarrow has\_car(A,B),open\_car(B)$ \\ 
    {\tt modeb(double(+car))} & $p(A) \leftarrow has\_car(A,B),double(B)$ \\ 
    {\tt modeb(jagged(+car))} & $p(A) \leftarrow has\_car(A,B),jagged(B)$ \\ 
    {\tt modeb(shape(+car,\#shape))} & $p(A) \leftarrow has\_car(A,B),shape(B,u\_shaped)$ \\ 
    {\tt modeb(load(+car,\#shape,\#int))} & $p(A) \leftarrow has\_car(A,B),shape(B,rectangle)$ \\ 
    {\tt modeb(wheels(+car,\#int))} & $p(A) \leftarrow has\_car(A,B),shape(B,hexagon)$ \\ 
    {\tt modeb(has\_car(+tr,-car))} & $p(A) \leftarrow has\_car(A,B),wheels(B,3)$ \\ 
     & $p(A) \leftarrow has\_car(A,B),wheels(B,2)$ \\ 
     & $p(A) \leftarrow has\_car(A,B),load(B,circle,3)$ \\ 
     & $p(A) \leftarrow has\_car(A,B),load(B,rectangle,3)$ \\ 
     & $p(A) \leftarrow has\_car(A,B),load(B,hexagon,3)$ \\ 
     & $p(A) \leftarrow has\_car(A,B),load(B,circle,2)$ \\ 
     & \vdots \\
    \hline
    \end{tabular}
    \caption{Examples of mode-definitions and simple feature-clauses for the Trains problem.}
    \label{fig:trainmodes}
\end{figure}

\begin{figure}[!htb]
    \centering
    \begin{tabular}{ll}
    \hline
    \multicolumn{1}{c}{Modes} & \multicolumn{1}{c}{$\Mu$-simple feature clauses}\\ 
    \hline
    {\tt modeh(p((+wkfile,+wkrank,+wrfile,} & $p((A,B,C,D,E,F)) \leftarrow C=E$ \\
    ~~~~~~~~{\tt +wrrank,+bkfile,+bkrank)))} & $p((A,B,C,D,E,F)) \leftarrow B=D$ \\
    {\tt modeb(lt(+wkrank,+wrrank))} & $p((A,B,C,D,E,F)) \leftarrow B=F$ \\
    {\tt modeb(lt(+wkrank,+bkrank))} & $p((A,B,C,D,E,F)) \leftarrow D=F$ \\
    {\tt modeb(lt(+wrrank,+wkrank))} & $p((A,B,C,D,E,F)) \leftarrow adj(A,C)$ \\
    {\tt modeb(lt(+wrrank,+bkrank))} & $p((A,B,C,D,E,F)) \leftarrow adj(A,E)$ \\
    {\tt modeb(lt(+bkrank,+wkrank))} & $p((A,B,C,D,E,F)) \leftarrow adj(C,E)$ \\
    {\tt modeb(lt(+bkrank,+wrrank))} & $p((A,B,C,D,E,F)) \leftarrow adj(B,D)$ \\
    {\tt modeb(lt(+wkfile,+wrfile))} & $p((A,B,C,D,E,F)) \leftarrow adj(B,F)$ \\
    {\tt modeb(lt(+wkfile,+bkfile))} & $p((A,B,C,D,E,F)) \leftarrow adj(D,F)$ \\
    {\tt modeb(lt(+wrfile,+wkfile))} & $p((A,B,C,D,E,F)) \leftarrow lt(A,C)$ \\
    {\tt modeb(lt(+wrfile,+bkfile))} & $p((A,B,C,D,E,F)) \leftarrow lt(C,A)$ \\
    {\tt modeb(lt(+bkfile,+wkfile))} & $p((A,B,C,D,E,F)) \leftarrow lt(A,E)$ \\
    {\tt modeb(lt(+bkfile,+wrfile))} & $p((A,B,C,D,E,F)) \leftarrow lt(E,A)$ \\
    {\tt modeb(adj(+wkrank,+wrrank))} & $p((A,B,C,D,E,F)) \leftarrow lt(C,E)$ \\
    {\tt modeb(adj(+wkrank,+bkrank))} & $p((A,B,C,D,E,F)) \leftarrow lt(E,C)$ \\
    {\tt modeb(adj(+wrrank,+bkrank))} & $p((A,B,C,D,E,F)) \leftarrow lt(B,D)$ \\
    {\tt modeb(adj(+wkfile,+wrfile))} & $p((A,B,C,D,E,F)) \leftarrow lt(D,B)$ \\
    {\tt modeb(adj(+wkfile,+bkfile))} & $p((A,B,C,D,E,F)) \leftarrow lt(B,F)$ \\
    {\tt modeb(adj(+wrfile,+bkfile))} & $p((A,B,C,D,E,F)) \leftarrow lt(F,B)$ \\
    {\tt modeb((+wkrank = +wrrank))} & $p((A,B,C,D,E,F)) \leftarrow lt(D,F)$ \\
    {\tt modeb((+wkrank = +bkrank))} & $p((A,B,C,D,E,F)) \leftarrow lt(F,D)$ \\
    {\tt modeb((+wrrank = +bkrank))} & $p((A,B,C,D,E,F)) \leftarrow A=E$ \\
    {\tt modeb((+wkfile = +wrfile))} & $p((A,B,C,D,E,F)) \leftarrow C=E$ \\
    {\tt modeb((+wkfile = +bkfile))} & $p((A,B,C,D,E,F)) \leftarrow B=D$ \\
    {\tt modeb((+wrfile = +bkfile))} & $p((A,B,C,D,E,F)) \leftarrow B=F$ \\
     & $p((A,B,C,D,E,F)) \leftarrow D=F$ \\
     & $p((A,B,C,D,E,F)) \leftarrow adj(A,C)$ \\
     & $p((A,B,C,D,E,F)) \leftarrow adj(A,E)$ \\
     & $p((A,B,C,D,E,F)) \leftarrow adj(C,E)$ \\
     & \vdots \\
     \hline
    \end{tabular}
    \caption{Examples of mode-definitions and simple feature-clauses for the Chess problem.}
    \label{fig:chessmodes}
\end{figure}

\begin{figure}[!htb]
    \centering
    \begin{sideways}
    \begin{tabular}{ll}
    \hline
    \multicolumn{1}{c}{Modes} & \multicolumn{1}{c}{$\Mu$-simple feature clauses}\\    \hline
    {\tt modeh(class(+mol,+class))} & $p(A) \leftarrow has\_struc(A,B,C,amine),gteq(C,1)$ \\
    {\tt modeb(symbond(+mol,+atomid,+atomid,\#bondtype))} & $p(A) \leftarrow has\_struc(A,B,C,ammonium\_ion),gteq(C,1)$ \\
    {\tt modeb(bond(+mol,-atomid\_1,-atomid\_2,} & $p(A) \leftarrow has\_struc(A,B,C,benzene\_ring),gteq(C,1)$ \\
    ~~~~~~~~{\tt \#atomtype,\#atomtype,\#bondtype))} & $p(A) \leftarrow has\_struc(A,B,C,benzene\_ring),gteq(C,2)$ \\
    {\tt modeb(atom(+mol,-atomid,\#element))} & $p(A) \leftarrow has\_struc(A,B,C,benzene\_ring),gteq(C,3)$ \\
    {\tt modeb(has\_struc(+mol,-atomids,-length,\#structype))} & $p(A) \leftarrow has\_struc(A,B,C,benzene\_ring),gteq(C,4)$ \\
    {\tt modeb(connected(+mol,+atomids,+atomids))} & $p(A) \leftarrow has\_struc(A,B,C,benzene\_ring),gteq(C,5)$ \\
    {\tt modeb(fused(+mol,+atomids,+atomids))} & $p(A) \leftarrow has\_struc(A,B,C,benzene\_ring),gteq(C,6)$ \\
    {\tt modeb(gteq(+length,\#length))} & $p(A) \leftarrow has\_struc(A,B,C,conjug\_base\_car),gteq(C,1)$ \\
    {\tt modeb(lteq(+length,\#length))} & $p(A) \leftarrow has\_struc(A,B,C,conjug\_base\_car),gteq(C,2)$ \\
    {\tt modeb((+atomid = +atomid))} & $p(A) \leftarrow has\_struc(A,B,C,conjug\_base\_car),gteq(C,3)$ \\
    {\tt modeb((+atomidids = +atomidids))} & $p(A) \leftarrow has\_struc(A,B,C,diazo\_group),gteq(C,1)$ \\
    {\tt modeb((+length = +length))} & $p(A) \leftarrow has\_struc(A,B,C,diazo\_group),gteq(C,2)$ \\
     & $p(A) \leftarrow has\_struc(A,B,C,dithio\_ester\_car),gteq(C,1)$ \\
     & $p(A) \leftarrow has\_struc(A,B,C,dithio\_ester\_car),gteq(C,2)$ \\
     & $p(A) \leftarrow has\_struc(A,B,C,dithio\_ester\_car),gteq(C,3)$ \\
     & \vdots \\
     \hline
    \end{tabular}
    \end{sideways}
    \caption{Examples of mode-definitions and simple feature-clauses for the NCI problem.}
    \label{fig:ncimodes}
\end{figure}

\subsection{Debugging Inconsistent Explanations from the CRM}
\label{app:debug}

For the Chess problem, $1044$ instances (out of $10,000$) are inconsistently explained.
That is, the explanation graph from the most relevant output vertex does not contain
an acceptable feature-clause. Further examination reveals: (a) for all 1044 instances,
the predictions made by the CRM are correct; (b) a majority ($1033/1044$) of the inconsistently
instances are $+$ examples for which
the White King and Black King are on adjacent files and ranks (the corresponding
acceptable feature-clause is $p((A,B,C,D,E,F))$ $\leftarrow$
$adj(A,E),$ $adj(B,F)$). We find there are $19$ distinct `buggy explanations'
produced by the CRM for the inconsistently explained data. The corresponding
feature-clauses at the `root' in the explanation graph
are listed in \autoref{fig:buggyex} (for simplicity, we do not show the simpler
features and the full graph structure):

\begin{figure}[!htb]
    \centering
    \begin{tabular}{rl}
    \hline
    1. & $p((A,B,C,D,E,F)) \leftarrow adj(B,F)),lt(D,B),adj(C,E)$ \\
    2. & $p((A,B,C,D,E,F)) \leftarrow adj(D,F)),B=D,lt(A,E)$ \\
    3. & $p((A,B,C,D,E,F)) \leftarrow adj(A,E),lt(A,C),lt(F,D)$ \\
    4. & $p((A,B,C,D,E,F)) \leftarrow lt(E,C),A=C,lt(B,F))$ \\
    5. & $p((A,B,C,D,E,F)) \leftarrow lt(E,A),A=C,adj(A,C)$ \\
    6. & $p((A,B,C,D,E,F)) \leftarrow lt(E,A),lt(D,B),adj(A,E)$ \\
    7. & $p((A,B,C,D,E,F)) \leftarrow B=D,A=C,adj(A,E)$ \\
    8. & $p((A,B,C,D,E,F)) \leftarrow B=F,A=E,lt(E,C)$ \\
    9. & $p((A,B,C,D,E,F)) \leftarrow lt(A,C),lt(B,D),adj(A,C)$ \\
    10. & $p((A,B,C,D,E,F)) \leftarrow lt(D,F)),A=E,lt(E,C)$ \\
    11. & $p((A,B,C,D,E,F)) \leftarrow A=E,adj(D,F)),adj(B,F))$ \\
    12. & $p((A,B,C,D,E,F)) \leftarrow lt(B,F)),lt(E,C),adj(D,F))$ \\
    13. & $p((A,B,C,D,E,F)) \leftarrow adj(B,D),lt(E,A),adj(A,C)$ \\
    14. & $p((A,B,C,D,E,F)) \leftarrow adj(A,E),lt(C,E),adj(B,F))$ \\ 
    15. & $p((A,B,C,D,E,F)) \leftarrow A=E,A=C,adj(B,F))$ \\
    16. & $p((A,B,C,D,E,F)) \leftarrow lt(A,C),lt(F,B),adj(B,F))$ \\ 
    17. & $p((A,B,C,D,E,F)) \leftarrow lt(D,B),lt(A,E),adj(A,E)$ \\
    18. & $p((A,B,C,D,E,F)) \leftarrow lt(E,A),A=C,adj(B,D)$ \\
    19. & $p((A,B,C,D,E,F)) \leftarrow adj(A,E),B=F,lt(A,E)$ \\
    \hline
    \end{tabular}
    \caption{Buggy explanations produced by the CRM for the inconsistently explained 
    data instances from the Chess problem.}
    \label{fig:buggyex}
\end{figure}

\end{appendices}

\clearpage

\bibliography{references}
\end{document}